\documentclass[a4paper,12pt,numbers=noenddot,parskip=no]{scrartcl}
\usepackage[top=2cm, bottom=4cm]{geometry}

\usepackage[T1]{fontenc}
\usepackage[utf8]{inputenc}
\usepackage[greek,english]{babel}
\usepackage{lmodern}
\usepackage[euler]{textgreek}
\usepackage{microtype}
\usepackage{textcomp}

\setkomafont{sectioning}{\normalfont \rmfamily \bfseries}%
\setkomafont{paragraph}{\normalfont \rmfamily \bfseries}%

\usepackage{hyperref}%
\hypersetup{%
  breaklinks=true,%
  linktocpage=true,
  colorlinks=false,%
  linkcolor=black,%
  citecolor=black,%
  filecolor=black,%
  urlcolor=black,%
  pdftitle={},%
  pdfsubject={On importance-weighted autoencoders},%
  pdfauthor={Axel Finke},%
  pdfkeywords={},%
  pdfproducer={}%
  }%
\usepackage{url}

\usepackage{graphicx}
\graphicspath{{./../../input/fig/}}

\usepackage[labelfont={footnotesize, bf}, textfont=footnotesize, labelformat=simple, labelsep=period, aboveskip=7pt, belowskip=0pt]{caption}
\usepackage[labelfont={footnotesize, bf}, textfont=footnotesize, labelformat=simple, labelsep=period, aboveskip=4pt, belowskip=3pt]{subcaption}

\usepackage{amsfonts}
\usepackage{amsmath,bm}
\usepackage{mathtools} 
\mathtoolsset{showonlyrefs=true}
\usepackage[ntheorem, framemethod=TikZ]{mdframed} 
\usepackage{framed}  
\usepackage[amsmath, hyperref, framed]{ntheorem}%
\usepackage{mathdots}
\usepackage{amssymb}
\usepackage{nicefrac}       

\usepackage[overload]{empheq}

\newcommand{\normConst}{\mathcal{Z}} 

\DeclareMathOperator{\dN}{N} 

\newcommand{\diff}{\mathrm{d}} 
\newcommand{\intDiff}{\,\mathrm{d}} 
\DeclareMathOperator{\E}{\mathbb{E}} 
\DeclareMathOperator{\var}{var} 
\DeclareMathOperator{\ind}{\mathbf{1}} 
\newcommand{\T}{\mathrm{T}} 
\DeclareMathOperator{\bo}{\mathcal{O}}%
\DeclareMathOperator{\diag}{diag}%
\DeclareMathOperator{\KL}{\mathrm{D}_{\mathsc{kl}}}%
\newcommand{\iMat}{\mathrm{I}} 

\newcommand{\unitFun}{\mathbf{1}}%
\newcommand{\eul}{\mathrm{e}}%
\newcommand{\iidSim}{\mathrel{\overset{\mathsc{iid}}{\sim}}}

\newcommand{\calL}{\mathcal{L}}%
\newcommand{\mathsc}[1]{{\normalfont\textsc{#1}}}%

\newcommand{\IWAE}{\mathsc{iwae}}%
\newcommand{\RWS}{\mathsc{rws}}%

\newcommand{\ML}{\mathsc{ml}}%

\newcommand{\reals}{\mathbb{R}}%
\newcommand{\naturals}{\mathbb{N}}%
\newcommand{\spaceE}{\mathsf{E}}%
\newcommand{\spaceZ}{\mathsf{Z}}%

\newcommand*{\mycmss}{\fontfamily{cmss}\selectfont}
\DeclareTextFontCommand{\textcmss}{\mycmss}

\DeclareFontFamily{U}{mathx}{\hyphenchar\font45}
\DeclareFontShape{U}{mathx}{m}{n}{<-> mathx10}{}
\DeclareSymbolFont{mathx}{U}{mathx}{m}{n}
\DeclareMathAccent{\widebar}{0}{mathx}{"73}

\DeclareMathOperator{\argmin}{arg\,min}%
\DeclareMathOperator{\argmax}{arg\,max}%

\usepackage{graphicx,accents}

\makeatletter
\DeclareRobustCommand{\cev}[1]{%
  \mathpalette\do@cev{#1}%
}
\newcommand{\do@cev}[2]{%
  \fix@cev{#1}{+}%
  \reflectbox{$\m@th#1\vec{\reflectbox{$\fix@cev{#1}{-}\m@th#1#2\fix@cev{#1}{+}$}}$}%
  \fix@cev{#1}{-}%
}
\newcommand{\fix@cev}[2]{%
  \ifx#1\displaystyle
    \mkern#23mu
  \else
    \ifx#1\textstyle
      \mkern#23mu
    \else
      \ifx#1\scriptstyle
        \mkern#22mu
      \else
        \mkern#22mu
      \fi
    \fi
  \fi
}
\makeatother

\makeatletter
\def\theorem@checkbold{}
\makeatother

\newcommand{\qedwhite}{\hfill \ensuremath{\Box}}

\usepackage{scalerel,stackengine}
\makeatletter
\stackMath
\newcommand\reallywidehat[1]{%
\savestack{\tmpbox}{\stretchto{%
  \scaleto{%
    \scalerel*[\widthof{\ensuremath{#1}}]{\kern.1pt\mathchar"0362\kern.1pt}%
    {\rule{0ex}{\textheight}}
  }{\textheight}%
}{2.4ex}}%
\stackon[-6.9pt]{#1}{\tmpbox}%
}
\makeatother

\newcommand{\target}{\pi}

\renewcommand{\normConst}{\mathcal{Z}}

\newcommand{\uTarget}{\gamma}

\renewcommand{\RWS}{\textsc{rws}}
\newcommand{\RWSDREG}{\textsc{rws-dreg}}

\renewcommand{\IWAE}{\textsc{iwae}}
\newcommand{\IWAEDREG}{\textsc{iwae-dreg}}
\newcommand{\IWAESTL}{\textsc{iwae-stl}}

\newcommand{\AISLE}{\textsc{aisle}}
\newcommand{\AISLEKL}{\textsc{aisle-kl}}
\newcommand{\AISLEKLNOREP}{\textsc{aisle-kl-norep}}
\newcommand{\AISLECHISQ}{\textsc{aisle-$\chi^2$}}
\newcommand{\AISLECHISQNOREP}{\textsc{aisle-$\chi^2$-norep}}

\renewcommand{\KL}{\mathop{\mathrm{KL}}}
\newcommand{\chisq}{\mathop{\chi^2}}
\newcommand{\Div}{\mathop{\mathrm{D}}}
\newcommand{\divFun}{\text{\textflorin}}

\renewcommand{\unitFun}{1}

\newcommand{\particleReparametrised}{e}
\newcommand{\particles}{\mathbf{z}}

\newcommand{\particlesDistributed}{\mathbf{z} \sim q_\phi^{\otimes K}}

\usepackage[inline]{enumitem}

\makeatletter
\newcommand{\mylabel}[2]{#2\def\@currentlabel{#2}\label{#1}}
\makeatother

\qedsymbol{\ensuremath{\Box}}
\theoremstyle{plain}%

\theoremheaderfont{\normalfont\bfseries\rmfamily} 
\theorembodyfont{\normalfont\slshape}%
\theoremseparator{.}

\newtheorem{proposition}{Proposition}%
\newtheorem{lemma}{Lemma}%
\newtheorem{remark}{Remark}%

\theorembodyfont{\normalfont}%
\newmdtheoremenv[
 ntheorem=true,
 skipbelow = .6\baselineskip plus 1ex minus 1ex,
 skipabove = .6\baselineskip plus 1ex minus 1ex,
 innerleftmargin = 0pt,
 innerrightmargin = 0pt,
 leftline = false,
 rightline = false,
 needspace = 5ex 
]{framedAlgorithm}[theorem]{Algorithm}

\theoremstyle{nonumberplain}%
\theorembodyfont{\upshape \rmfamily}%
\theoremheaderfont{\upshape \rmfamily\bfseries}%
\theoremsymbol{\ensuremath{_\Box}}
\newtheorem{proof}{Proof}%

\usepackage{ifthen}%
\usepackage{mfirstuc}
\usepackage{xkeyval}%
\usepackage{xfor}%
\usepackage{amsgen}%
\usepackage{etoolbox}%
\usepackage{datatool-base}%

\usepackage[%
  xindy,%
  style=long,%
  toc=true,%
  nonumberlist,%
  acronym,%
  acronymlists={abbreviations},%
  hyperfirst=true
  ]{glossaries}%
  
\newacronym{PMCMC}{PMCMC}{particle Markov chain Monte Carlo}%
\newacronym{EMCMC}{EMCMC}{ensemble Markov chain Monte Carlo}%
\newacronym{APF}{APF}{auxiliary particle filter}%
\newacronym{FAAPF}{FA-APF}{fully-adapted auxiliary particle filter}%
\newacronym{PF}{PF}{particle filter}%
\newacronym{EHMM}{EHMM}{embedded hidden Markov models}%
\newacronym{HMM}{HMM}{hidden Markov model}%
\newacronym{CPF}{CPF}{conditional particle filter}%
\newacronym{CPFBS}{CPF-BS}{conditional particle filter with backward sampling}%
\newacronym{CPFAS}{CPF-AS}{conditional particle filter with backward sampling}%
\newacronym{PG}{PG}{particle Gibbs}%
\newacronym{PGBS}{PG-BS}{particle Gibbs sampler with backward sampling}%
\newacronym{PGAS}{PG-AS}{particle Gibbs sampler with backward sampling}%
\newacronym{SQMC}{SQMC}{sequential quasi Monte Carlo}%
\newacronym{RQMC}{RQMC}{randomised quasi Monte Carlo}%
\newacronym[user1={ancestor-sampling}]{AS}{AS}{ancestor sampling}%
\newacronym[user1={backward-sampling}]{BS}{BS}{backward sampling}%
\newacronym{PDF}{PDF}{probability density function}%
\newacronym{IID}{IID}{independent and identically distributed}%
\newacronym{MCMC}{MCMC}{Markov chain Monte Carlo}%
\newacronym{MH}{MH}{Metropolis--Hastings}%
\newacronym{ESS}{ESS}{effective sample size}%
\newacronym{PMMH}{PMMH}{particle marginal Metro\-po\-lis--Has\-tings}%
\newacronym{MCWM}{MCWM}{Monte Carlo within Metropolis}%
\newacronym{CDF}{CDF}{cumulative distribution function}%
\newacronym{SMC}{SMC}{sequential Monte Carlo}%
\newacronym{CSMC}{CSMC}{conditional sequential Monte Carlo}%
\newacronym{EPSRC}{EPSRC}{Engineering and Physical Sciences Research Council}%
\newacronym{LW}{LW}{Liu~\&~West}%
\newacronym{PL}{PL}{particle learning}%
\newacronym{FF}{FF}{fertility factor}%
\newacronym{CLT}{CLT}{central limit theorem}%
\newacronym{BPF}{BPF}{bootstrap particle filter}%
\newacronym{QMC}{QMC}{quasi Monte Carlo}%
\newacronym{IACT}{IACT}{integrated autocorrelation time}%
\newacronym{CI}{CI}{confidence interval}%
\newacronym{ATI}{ATI}{Alan Turing Institute}%
\newacronym{SA}{SA}{simulated annealing}%
\newacronym[user1={importance-sampling}]{IS}{IS}{importance sampling}%
\newacronym{MSE}{MSE}{mean-square error}%
\newacronym{MAP}{MAP}{maximum a-posteriori}%
\newacronym{ABC}{ABC}{approximate Bayesian computation}%
\newacronym{NCP}{NCP}{non-centred parametrisation}%
\newacronym{CP}{CP}{centred parametrisation}%
\newacronym[\glslongpluralkey={Markov jump processes}]{MJP}{MJP}{Markov jump process}%
\newacronym{ELBO}{ELBO}{evidence lower bound}%
\newacronym{KL}{KL}{Kullback--Leibler}%
\newacronym{IWAE}{IWAE}{importance weighted autoencoder}%
\newacronym{VAE}{VAE}{variational autoencoder}%
\newacronym{VIS}{VIS}{variational importance sampling}
\newacronym{VSMC}{VSMC}{variational sequential Monte Carlo}
\newacronym{VSMCALT}{VSMC-ALT}{alternative variational sequential Monte Carlo}
\newacronym{RWS}{RWS}{reweighted wake-sleep}
\newacronym{RMSE}{RMSE}{root mean-square error}%
\newacronym{SAME}{SAME}{state augmentation for marginal estimation}%
\newacronym{MLE}{MLE}{maximum-likelihood estimate}%
\newacronym{ML}{ML}{maximum-likelihood}%

\newacronym[\glslongpluralkey={MCMC particle filters}]{MCMCPF}{MCMC-PF}{MCMC particle filter}%
\newacronym[\glslongpluralkey={MCMC bootstrap particle filters}]{MCMCBPF}{MCMC-BPF}{MCMC bootstrap particle filter}%
\newacronym[\glslongpluralkey={MCMC fully-adapted particle filters}]{MCMCFAAPF}{MCMC-FA-APF}{MCMC fully-adapted auxiliary particle filter}%
\newacronym{MCMCAPF}{MCMC-APF}{MCMC auxiliary particle filters}%
\newacronym{FAAPFalt}{FA-APF}{fully-adapted auxiliary PF}%
\newacronym{BPFalt}{BPF}{bootstrap PF}%
\newacronym{WLLN}{WLLN}{weak law of large numbers}%
\newacronym{SLLN}{SLLN}{strong law of large numbers}%
\newacronym{HMC}{HMC}{Hamiltonian Monte Carlo}%
\newacronym{MALA}{MALA}{Metropolis-adjusted Langevin algorithm}%
\newacronym{AIS}{AIS}{adaptive importance sampling}%

\newacronym[user1={sticking-the-landing (STL)}]{STL}{STL}{sticking the landing}%
\newacronym{ADAM}{ADAM}{adaptive moment estimation}%

\newacronym{TMC}{TMC}{tensor Monte Carlo}

\newacronym{SMCSF}{SMC*\!}{`subsample-free' SMC}%
\newacronym{VSMCSF}{VSMC*\!}{`subsample-free' VSMC}%
\newacronym{IWAESTL}{IWAE-STL}{`sticking-the-landing' IWAE}%
\newacronym{IWAEDREG}{IWAE-DREG}{`doubly-reparametrised' IWAE}%
\newacronym{RWSDREG}{RWS-DREG}{`doubly-reparametrised' RWS}%
\newacronym{AISLE}{AISLE}{adaptive importance sampling for learning}
\newacronym{AISLEKLNOREP}{AISLE-KL-NOREP}{???}%
\newacronym{AISLEKL}{AISLE-KL}{???}%
\newacronym{AISLECHISQNOREP}{AISLE-$\chi^2$-NOREP}{???}%
\newacronym{AISLECHISQ}{AISLE-$\chi^2$}{???}%
\newacronym{VSMCALTSTL}{VSMC-ALT-STL}{???}%
\newacronym{VSMCALTDREG}{VSMC-ALT-DREG}{???}%
\newacronym{VSMCSTLSF}{VSMC*-STL}{???}%
\newacronym{VSMCDREGSF}{VSMC*-DREG}{???}%
\newacronym[\glslongpluralkey={adaptive particle filters for efficient learning}]{APFLE}{APFLE}{adaptive particle filter for learning}
\newacronym{APFLEKLNOREP}{APFLE-KL-NOREP}{???}%
\newacronym{APFLEKL}{APFLE-KL}{???}%
\newacronym{APFLECHISQNOREP}{APFLE-$\chi^2$-NOREP}{???}%
\newacronym{APFLECHISQ}{APFLE-$\chi^2$}{???}%
\newacronym{APFLESF}{APFLE*\!}{`subsample-free' APFLE}%
\newacronym{APFLEKLNOREPSF}{APFLE*-KL-NOREP}{???}%
\newacronym{APFLEKLSF}{APFLE*-KL}{???}%
\newacronym{APFLECHISQNOREPSF}{APFLE*-$\chi^2$-NOREP}{???}%
\newacronym{APFLECHISQSF}{APFLE*-$\chi^2$}{???}%
\usepackage{csquotes}%
\usepackage{natbib}
\setcitestyle{authoryear}
\title{On importance-weighted autoencoders}
\author{Axel Finke \and Alexandre H.~Thi\'ery}
\date{Department of Probability and Applied Statistics,\\National University of Singapore\\[0.4cm]\today}

\begin{document}

\maketitle
\glsunset{IID}
\glsunset{ADAM}

\begin{abstract}
\noindent{}The \emph{\gls{IWAE}} \citep{burda2016importance} is a popular variational-inference method which achieves a tighter evidence bound (and hence a lower bias) than standard \glsdescplural{VAE} by  optimising a \emph{multi-sample objective,} i.e.\ an objective that is expressible as an integral over $K>1$ Monte Carlo samples. Unfortunately, \gls{IWAE} crucially relies on the availability of reparametrisations and even if these exist, the multi-sample objective leads to inference-network gradients which break down as $K$ is increased \citep{rainforth2018tighter}. This breakdown can only be circumvented by removing high-variance score-function terms, either by heuristically ignoring them (which yields the \emph{\gls{IWAESTL}} gradient from \citet{roeder2017sticking}) or through an identity from \citet{tucker2019reparametrised} (which yields the \emph{\gls{IWAEDREG}} gradient). In this work, we argue that directly optimising the proposal distribution in importance sampling as in the \emph{\gls{RWS}} algorithm from \citet{bornschein2015reweighted} is preferable to optimising \gls{IWAE}-type multi-sample objectives. To formalise this argument, we introduce an adaptive-importance sampling framework termed \emph{\gls{AISLE}} which slightly generalises the \gls{RWS} algorithm. We then show that \gls{AISLE} admits \gls{IWAESTL} and \gls{IWAEDREG} (i.e.\ the \gls{IWAE}-gradients which avoid breakdown) as special cases.
%
\end{abstract}

\glsunset{RWSDREG}
\glsunset{AISLEKLNOREP}
\glsunset{AISLEKL}
\glsunset{AISLECHISQNOREP}
\glsunset{AISLECHISQ}

\glsreset{IID}
\glsreset{AISLE}
\glsreset{IWAE}
\glsreset{RWS}

\section{Introduction}

\subsection{Problem statement}
Let $x$ be some observation and let $z$ be some latent variable taking values in some space $\spaceZ$. These are modeled via the \emph{generative model} $p_\theta(z,x) = p_\theta(z) p_\theta(x|z)$ which gives rise to the marginal likelihood $p_\theta(x) = \int_\spaceZ p_\theta(z,x) \intDiff z$ of the model parameters $\theta$. The latter may also be viewed as the evidence for the model parametrised by a particular value of $\theta$. In this work, we analyse algorithms for \emph{variational inference,} i.e. algorithms which aim to
\begin{enumerate}
 \item\label{enum:goal:1} learn the generative model, i.e.\ find a value $\theta^\star$ which is approximately equal to the \emph{\gls{MLE}} $\smash{\theta^\ML \coloneqq \argmax_\theta p_\theta(x)}$;

 \item\label{enum:goal:2} construct a tractable \emph{variational approximation} $q_{\phi,x}(z)$ of 
 $p_\theta(z|x) = p_\theta(z,x)/p_\theta(x)$, 
 i.e.\ find the value $\phi^\star$ such that $q_{\phi^{\mathrlap{\star}}\,,x}(z)$ is as close as possible to $p_\theta(z|x)$ in some suitable sense.
\end{enumerate}
A few comments about this setting are in order. Firstly, as is common in the literature, we restrict our presentation to a single latent representation--observation pair $(z,x)$ to avoid notational clutter -- the extension to multiple independent observations is straightforward.  Secondly, we assume that no parameters are shared between the generative model $p_\theta(z,x)$ and the variational approximation $q_{\phi,x}(z)$. This is common in neural-network applications but could be relaxed. Thirdly, our setting is general enough to cover amortised inference which is why we often refer to $\phi$ as the parameters of an \emph{inference network.} 

In recent years, two classes of stochastic-gradient ascent algorithms for optimising $(\theta, \phi)$ -- which employ $K \geq 1$ Monte Carlo samples (`particles') to reduce errors -- have been proposed. 
\begin{itemize} \glsunset{IWAE}
 \item \textbf{\gls{IWAE}.} \glsreset{IWAE} The \emph{\gls{IWAE}} \citep{burda2016importance} optimises a joint objective for $\theta$ and $\phi$ (which is `biased' for $\theta$ though optimising $\phi$ or increasing $K$ decreases this bias) whose gradients are unbiasedly approximated via the Monte Carlo method. Unfortunately, as this \emph{multi-sample} objective is expressible as an integral on a $K$-dimensional space, the signal-to-noise ratio of the \gls{IWAE} $\phi$-gradient vanishes as $K$ grows \citep{rainforth2018tighter}. 
 Two modified \gls{IWAE} $\phi$-gradients avoid this breakdown by removing high-variance `score-function' terms:
 \begin{itemize} \glsunset{IWAESTL}
    \item \textbf{\gls{IWAESTL}.} \glsreset{IWAESTL} The \emph{\gls{IWAESTL}} $\phi$-gradient \citep{roeder2017sticking} heuristically drops the problematic score-function terms from the \gls{IWAE} $\phi$-gradient. This induces bias for the \gls{IWAE} objective. 
    
    \glsunset{IWAEDREG}
    \item \textbf{\gls{IWAEDREG}.} \glsreset{IWAEDREG} The \emph{\gls{IWAEDREG}} $\phi$-gradient  \citep{tucker2019reparametrised} unbiasedly removes the problematic score-function terms from the \gls{IWAE} $\phi$-gradient using a formal identity.
 \end{itemize}
 
 \glsunset{RWS}
 \item \textbf{\gls{RWS}.} \glsreset{RWS} The \emph{\gls{RWS}} algorithm \citep{bornschein2015reweighted} 
 optimises two separate but `unbiased' objectives for $\theta$ and $\phi$. Its gradients are approximated by self-normalised importance sampling with $K$ particles which induces bias (though again, optimising $\phi$ or increasing $K$ decreases this bias). \gls{RWS} can be viewed as an adaptive importance-sampling approach which iteratively improves its proposal distribution while simultaneously optimising $\theta$ via stochastic approximation. Crucially, \gls{RWS} is \emph{not} a multi-sample objective approach and hence does not require continuous reparametrisations nor do its $\phi$-gradients suffer from the breakdown highlighted in \citet{rainforth2018tighter}.
\end{itemize}

Of these two methods, the \gls{IWAE} is the most popular and \citet{tucker2019reparametrised} demonstrated empirically that \gls{RWS} can break down, conjecturing that this is due to the fact that \gls{RWS} does not optimise a joint objective (for $\theta$ and $\phi$). Meanwhile, the \gls{IWAESTL} gradient performed consistently well despite lacking a firm theoretical footing. Yet, \gls{IWAE} suffers from the above-mentioned $\phi$-gradient breakdown and exhibited inferior empirical performance to \gls{RWS} in some scenarios \citep{le2019revisiting}. Thus, it is not clear whether the multi-sample objective approach of \gls{IWAE} or the adaptive importance-sampling approach of \gls{RWS} is preferable.

In this work, we argue that the adaptive importance-sampling paradigm of \gls{RWS} is preferable to the multi-sample objective paradigm of \glspl{IWAE}. This is because 
\begin{enumerate*}[label=(\alph*)]
 \item the multi-sample objective crucially requires reparametrisations and, even if these are available, leads to the $\phi$-gradient breakdown,
 \item modifications of the \gls{IWAE} $\phi$-gradient which avoid this breakdown (i.e.\ \gls{IWAESTL} and \gls{IWAEDREG}) can be justified in a more principled manner by taking an \gls{RWS}-type adaptive importance-sampling view.
\end{enumerate*}

To formalise these arguments, we slightly generalise the \gls{RWS} algorithm to obtain a generic adaptive importance-sampling framework for variational inference which we term \emph{\gls{AISLE}} for ease of reference. We then show that \gls{AISLE} admits not only \gls{RWS} but also the \gls{IWAEDREG} and \gls{IWAESTL} gradients as special cases.

\subsection{Contributions}

Importance sampling as well as the \gls{IWAE} and \gls{RWS} algorithms are reviewed in Section~\ref{sec:background}. Novel material is presented in Section~\ref{sec:aisle}, where we we introduce the \gls{AISLE}-framework:
\begin{itemize}
  \item 
  
  In Subsection~\ref{subsec:aisle:special_case_I}, we show that \gls{AISLE} admits \gls{RWS} as a special case. In addition, we prove that the \gls{IWAESTL} gradient is in turn recovered as a special case of \gls{RWS} (and hence of \gls{AISLE}) via a principled and novel application of the `double-reparametrisation' identity from \citet{tucker2019reparametrised}. This indicates that the breakdown of \gls{RWS} observed in \citet{tucker2019reparametrised} may not be due to its lack of a joint objective as previously conjectured (because \gls{IWAESTL} avoided this breakdown). Our work also provides a theoretical foundation for \gls{IWAESTL} which was hitherto only heuristically justified as a biased \gls{IWAE} gradient.
 
 \item 
 
 In Subsection~\ref{subsec:aisle:special_case_II}, we prove that \gls{AISLE} also admits the \gls{IWAEDREG} gradient as a special case. Our derivation also makes it clear that the learning rate should be scaled as $\bo(K)$ for the \gls{IWAE} $\phi$-gradient (and its modified version \gls{IWAEDREG}) unless the gradients are normalised as implicitly done by popular optimisers such as \gls{ADAM} \citep{kingma2015adam}. In contrast, the scaling of the learning rate for \gls{AISLE} is independent of $K$.
 
 \item In the supplementary materials, we provide some insight into the impact of the self-normalisation bias on some of the importance-sampling based gradient approximations (Appendix~\ref{app:sec:bias}) and empirically compare all algorithms discussed in this work (Appendix~\ref{app:sec:illustrations}). 
\end{itemize}

We stress that the point of our work is not to derive new algorithms nor to establish which of the various special cases of \gls{AISLE} is preferable. Indeed, while we compare all algorithms discussed in this work empirically on Gaussian models in the in the supplementary materials available with this paper, we refer the reader to \citet{tucker2019reparametrised, le2019revisiting} for a extensive empirical comparisons of all the algorithms discussed in this work. Instead, the main message of our work is that the \gls{AISLE}-type adaptive importance-sampling paradigm is preferable to the  \gls{IWAE}-type multi-sample objective paradigm because the former allows us to derive all the above-mentioned variants of \gls{IWAE} -- as well as further algorithms which do not require reparametrisations -- in a principled manner (the only exception is the standard \gls{IWAE} reparametrisation $\phi$-gradient but this variant suffers from the breakdown highlighted in \citet{rainforth2018tighter} and was therefore consistently outperformed by the other variants in the simulations shown in Appendix~\ref{app:sec:illustrations} and in \citet{tucker2019reparametrised}, for $K>1$).
%
%

\subsection{Notation}

We assume that all (probability) measures $p$ used in this work are absolutely continuous w.r.t.\ some suitable dominating measure $\diff z$ and with some abuse of notation, we use the same symbol for the measure and the density, i.e.\ we write $p(\diff z) = p(z) \diff z$. With this convention, we employ the shorthand $\smash{p(f) \coloneqq \int_\spaceZ f(z) p(z) \intDiff z}$ for the integral of some $p$-integrable test function $f$; thus, $\smash{p(f) = \E_{z \sim p}[f(z)]}$ if $p$ is a probability measure. Furthermore, $\smash{q^{\otimes K}(z^{1:K}) \coloneqq \prod_{k=1}^K q(z^k)}$. We also let $0$ denote vectors or matrices of $0$s of some appropriate size which will be clear from the context and we let $\unitFun$ be the function that takes value $1$ everywhere on its domain. To keep the notation concise, we hereafter suppress dependence on the observation $x$, i.e.\ we write $\smash{q_\phi(z) \coloneqq q_{\phi,x}(z)}$ as well as
\begin{equation}
 \target_{\theta}(z) \coloneqq p_\theta(z|x) = \frac{p_\theta(z,x)}{p_\theta(x)} = \frac{\uTarget_\theta(z)}{\normConst_\theta},
\end{equation}
where $\uTarget_{\theta}(z) \coloneqq p_\theta(z,x)$ and where $\normConst_\theta \coloneqq p_\theta(x) = \int_\spaceZ \uTarget_\theta(z) \intDiff z = \uTarget_\theta(\unitFun)$.

\section{Background}
\label{sec:background}

\subsection{Importance sampling}

\paragraph{Basic idea.}

We hereafter write $\psi \coloneqq (\theta, \phi)$ and assume that the support of $q_{\phi}$ includes the support of $\target_{\theta}$ so that the importance weight function $w_{\psi}(z) \coloneqq {\uTarget_{\theta}(z)} / {q_{\phi}(z)}$ is well defined. 
For $\target_\theta$-integrable $f\colon \spaceZ \to \reals$, we can unbiasedly approximate integrals of the form
\begin{equation}
 \uTarget_{\theta}(f) \coloneqq \int_\spaceZ f(z) \uTarget_{\theta}(z) \intDiff z = \int_\spaceZ f(z) w_\psi(z) q_\phi(z) \intDiff z = q_\phi(f w_\psi), \label{eq:unnormalised_target_integral}
\end{equation}
via \emph{importance sampling} using a set of $K$ particles, $\smash{\particles \coloneqq (z^{\mathrlap{1}}, \dotsc, z^K) \sim q_\phi^{\otimes K}}$, which are \gls{IID} according to $q_\phi$, as
\begin{equation}
  \hat{\uTarget}_{\theta}\langle \phi, \particles\rangle(f) \coloneqq \frac{1}{K}\sum_{k=1}^{\smash{K}} w_{\psi}(z^k) f(z^k).
\end{equation}
Here, the notation $\langle \phi, \particles\rangle$ stresses the dependence of the estimator on $\phi$ and $\particles$. Note that this is simply an application of the vanilla Monte Carlo method to the expectation from the r.h.s.\ of \eqref{eq:unnormalised_target_integral}. 
Hereafter, we use the convention that $\smash{\E = \E_{\particlesDistributed}}$ and $\smash{\var_{\particlesDistributed}}$ denote expectation and variance w.r.t.\ $\smash{\particles = (z^{\mathrlap{1}}, \dotsc, z^K) \sim q_\phi^{\otimes K}}$. 
%

\paragraph{Self-normalised importance sampling.} 
Approximating integrals of the form
\begin{equation}
 \target_{\theta}(f) \coloneqq \int_\spaceZ f(z) \target_{\theta}(z) \intDiff z = \frac{\uTarget_{\theta}(f)}{\uTarget_{\theta}(\unitFun)},
\end{equation}
is slightly more complicated because the marginal likelihood $\normConst_\theta = \uTarget_{\theta}(\unitFun) = p_\theta(x)$ is intractable. Plugging in importance-sampling approximations for both the numerator and denominator leads to the following \emph{self-normalised} importance sampling estimate:
\begin{equation}
  \hat{\target}_{\theta}\langle \phi, \particles\rangle(f) \coloneqq \frac{\hat{\uTarget}_{\theta}\langle \phi, \particles\rangle(f)}{\hat{\uTarget}_{\theta}\langle \phi, \particles\rangle(\unitFun)} =  \sum_{k=1}^K \frac{w_{\psi}(z^k)}{\sum_{l=1}^K w_{\psi}(z^l)} f(z^k).
\end{equation}

\paragraph{Properties.}
Proposition~\ref{prop:properties_of_importance_sampling} summarises some well-known properties of importance-sampling approximations \citep[see, e.g.,][]{geweke1989bayesian} used throughout this work.

\begin{proposition}\label{prop:properties_of_importance_sampling}
 Let $f \colon \spaceZ \to \reals$ be $\target_\theta$-integrable and $\smash{\particlesDistributed}$. Then if $\sup w_\psi < \infty$,\\[-4ex]
\begin{enumerate}
 \item \label{prop:properties_of_importance_sampling:1} $\smash{\E[\hat{\uTarget}_{\theta}\langle\phi, \particles\rangle(f)] = \uTarget_\theta(f)}$, for any $K \in \naturals$,
 \item \label{prop:properties_of_importance_sampling:2} $\smash{\E[\hat{\target}_\theta\langle\phi, \particles\rangle(f)] = \target_\theta(f) + \bo(K^{-1})}$ and 
 $\smash{\var[\hat{\target}_\theta\langle\phi, \particles\rangle(f)] = \bo(K^{-1})}$,
 \item \label{prop:properties_of_importance_sampling:3} $\smash{\hat{\uTarget}_\theta\langle\phi, \particles\rangle(f) \to \uTarget_\theta(f)}$ and $\smash{\hat{\target}_\theta\langle\phi, \particles\rangle(f) \to \target_\theta(f)}$, almost surely, as $\smash{K \to \infty}$.
\end{enumerate}
\end{proposition}
\begin{proof}
  Part~\ref{prop:properties_of_importance_sampling:1} is immediate; Part~\ref{prop:properties_of_importance_sampling:2} is proved, e.g.\ in \citet[p.~35]{liu2001monte}; Part~\ref{prop:properties_of_importance_sampling:3} is a direct consequence of the strong law of large numbers. $\qedwhite$
\end{proof}

Part~\ref{prop:properties_of_importance_sampling:1} of Proposition~\ref{prop:properties_of_importance_sampling} shows that (non self-normalised) importance-sampling approximations $\hat{\uTarget}_\theta\langle\phi, \particles\rangle(f)$ are unbiased. In particular,
\begin{equation}
 \widehat{\normConst}_\theta\langle\phi, \particles\rangle \coloneqq \hat{\uTarget}_\theta\langle\phi, \particles\rangle(\unitFun) = \frac{1}{K} \sum_{k=1}^K w_\psi(z^k),
\end{equation}
is an unbiased estimate of the normalising constant $\normConst_\theta = \uTarget_\theta(\unitFun) = p_\theta(x)$. In contrast, the self-normalised importance-sampling approximation $\hat{\target}_\theta\langle\phi, \particles\rangle(f)$ is typically biased. However, Part~\ref{prop:properties_of_importance_sampling:3} shows that it is still consistent and Part~\ref{prop:properties_of_importance_sampling:2} ensures that the bias decays quickly in $K$.

\glsreset{IWAE}
\subsection{\Glsentryfull{IWAE}} \label{subsec:iwae}
\glsreset{IWAE}
\glsreset{VAE}

\paragraph{Objective.}

The \emph{\gls{IWAE}}, introduced by \citet{burda2016importance}, seeks to find a value $\theta^\star$ of the generative-model parameters $\theta$ which maximises a lower bound $\smash{\calL_\psi^K}$ on the log-marginal likelihood (`evidence') which depends on the inference-network parameters $\phi$ and the number of samples, $K \geq 1$,
\begin{gather}
 \smash{\psi^\star \coloneqq (\theta^{\mathrlap{\star}}\,, \phi^\star) \coloneqq \argmax_{\psi} \calL_{\psi}^K,}\\
 \calL_{\psi}^K \coloneqq \E\bigl[\log \widehat{\normConst}_\theta\langle \phi, \particles\rangle \bigr]. \label{eq:iwae:objective}
\end{gather}
For any finite $K$, optimisation of the inference-network parameters $\phi$ tightens the evidence bound. \citet{burda2016importance} prove the following properties. Firstly, $\smash{\calL_\psi^K \leq \log \normConst_\theta}$ follows from Jensen's inequality and Part~\ref{prop:properties_of_importance_sampling:1} of Proposition~\ref{prop:properties_of_importance_sampling}. Secondly, again by Jensen's inequality, $\smash{\calL_\psi^K \leq \calL_\psi^{K+1}}$. These inequalities are strict unless $\target_{\theta} = q_{\phi}$. Finally, Part~\ref{prop:properties_of_importance_sampling:3} of Proposition~\ref{prop:properties_of_importance_sampling} (along with the dominated convergence theorem) shows that for any $\phi$, $\smash{\calL_\psi^K \uparrow \log \normConst_\theta}$ as $K \to \infty$. If $K = 1$, the \gls{IWAE} reduces to the \gls{VAE} from \citet{kingma2014auto}. However, for $K > 1$, as pointed out in \citet{cremer2017reinterpreting, domke2018importance}, the \gls{IWAE} also constitutes another \gls{VAE} on an extended space based on an auxiliary-variable construction developed in \citet{andrieu2009pseudo, andrieu2010particle, lee2011auxiliary} \citep[see, e.g.\@][for a review]{finke2015extended}.

\paragraph{Standard reparametrisation gradient.}
The gradient of the \gls{IWAE} objective from \eqref{eq:iwae:objective}
$\smash{
 \nabla_\psi \calL_{\psi}^K
= \E\bigl[  \nabla_\psi \log \widehat{\normConst}_\theta\langle \phi, \particles\rangle + G_\psi(\particles)\bigr]}$,
with $\smash{G_\psi(\particles) \coloneqq \log \widehat{\normConst}_\theta\langle \phi, \particles\rangle  \sum_{k=1}^{\smash{K}} \nabla_\psi \log q_\phi(z^k)}$, is typically intractable. However, it could be approximated unbiasedly via a vanilla Monte Carlo approximation using a single sample point $\smash{\particles = (z^{\mathrlap{1}}, \dotsc, z^K) \sim q_\phi^{\otimes K}}$. Unfortunately, the term $\smash{G_\psi(\particles)}$ typically has such a large variance that the Monte Carlo approximation becomes impracticably noisy \citep{paisley2012variational}. To remove this high-variance term, the well known \emph{reparametrisation trick} \citep{kingma2014auto} is usually employed. It requires that the following assumption holds. 
\begin{enumerate}[label=\textbf{(R\arabic*)}, ref=\textbf{R\arabic*}]
 \item \textsl{\label{as:proposal:reparametrisation} There exists a distribution $q$ on some space $\spaceE$ and a diffeomorphism $h_{\phi} \colon \spaceE \to \spaceZ$ such that $\particleReparametrised \sim q$ $\Leftrightarrow$ $ h_{\phi}(\particleReparametrised) \sim q_{\phi}$.}
\end{enumerate}
Under \ref{as:proposal:reparametrisation}, the gradient can alternatively be expressed as
\begin{align}
 \nabla_\psi \calL_{\psi}^K
 & =  \E_{\particleReparametrised^{\mathrlap{1}},\dotsc,\particleReparametrised^K \iidSim q}\bigl[\nabla_\psi \log \widehat{\normConst}_\theta\langle \phi, \{h_\phi(\particleReparametrised^k)\}_{k=1}^K \rangle \bigr] 
 \\
 & = 
 \E_{\particleReparametrised^{\mathrlap{1}},\dotsc,\particleReparametrised^K \iidSim q}\biggl[
 \sum_{k=1}^K \frac{w_{\psi}(h_{\phi}(\particleReparametrised^k))}{\sum_{l=1}^K w_{\psi}(h_{\phi}(\particleReparametrised^l))} 
  \nabla_\psi \log w_{\psi}(h_{\phi}(\particleReparametrised^k))
 \biggr]\\
 & = \E\biggl[
 \sum_{k=1}^K \frac{w_{\psi}(z^k)}{\sum_{l=1}^K w_{\psi}(z^l)} 
   \begin{pmatrix}
  \nabla_\theta \log \gamma_\theta(z^k)\\
   \blacktriangledown_\psi(z^k) - \nabla_\phi \log q_\phi(z^k)
  \end{pmatrix}
     \biggr], \label{eq:iwae:phi_gradient_exact}
\end{align}
with 
 \begin{equation}
  \blacktriangledown_\psi(z) \coloneqq \nabla_\phi [{\log} \circ {w_{\psi'}} \circ {h_{\phi}}]|_{\psi' = \psi}(h_{\phi}^{-1}(z)).
 \end{equation}
\gls{IWAE} then uses a vanilla Monte Carlo estimate of \eqref{eq:iwae:phi_gradient_exact} (using a single sample point $\smash{\particlesDistributed}$):
\begin{equation}
  \begin{bmatrix}
    \widehat{\nabla}_\theta^\IWAE\langle \phi, \particles\rangle\\
     \widehat{\nabla}_\phi^\IWAE\langle \theta, \particles\rangle
  \end{bmatrix}
  \coloneqq 
  \sum_{k=1}^{\smash{K}} \frac{w_{\psi}(z^k)}{\sum_{l=1}^K w_{\psi}(z^l)}
  \begin{bmatrix}
  \nabla_\theta \log \gamma_\theta(z^k)\\
   \blacktriangledown_\psi(z^k) - \nabla_\phi \log q_\phi(z^k)
  \end{bmatrix}.
  \label{eq:iwae:phi-gradient_approximation}
\end{equation}

\paragraph{$\phi$-gradient issues.} Before proceeding, we state the following lemma, proved in \citet[][Section~8.1]{tucker2019reparametrised}, which generalises of the well-known identity $q_\phi(\nabla_\phi \log q_\phi) = 0$.

\begin{lemma}[\citet{tucker2019reparametrised}] \label{lem:tucker} Under \ref{as:proposal:reparametrisation}, for suitably integrable $f_\psi \colon \spaceZ \to \reals$:
\begin{equation}
 q_\phi(f_\psi \nabla_\phi \log q_\phi) =  q_\phi({\nabla_\phi [{f_{\psi'}} \circ {h_\phi}]|_{\psi' = \psi}} \circ {h_\phi^{-1}}). \tag*{\qedwhite}
\end{equation}
\end{lemma}

We now exclusively focus on the $\phi$-portion of the \gls{IWAE} gradient, $\smash{\widehat{\nabla}_{\phi}^\IWAE\langle \theta, \particles \rangle}$. 


\begin{remark}[drawbacks of the \gls{IWAE} $\phi$-gradient] \label{rem:drawbacks} The gradient $\smash{\widehat{\nabla}_{\phi}^\IWAE\langle \theta, \particles \rangle}$ has three drawbacks. The last two of these are attributable to the `score-function' terms $\nabla_\phi \log q_\phi(z)$ in the $\phi$-gradient portion of \eqref{eq:iwae:phi-gradient_approximation}.\\[-3.5ex]
 \begin{itemize}
  \item \textbf{Reliance on reparametrisations.} A continuous reparametrisation \`a la \ref{as:proposal:reparametrisation} is necessary to remove the high-variance term $\smash{G_\psi(\particles)}$; this makes it difficult to use \gls{IWAE} for models with e.g.\ discrete latent variables $z$ \citep{le2019revisiting}.
  
  \item \textbf{Vanishing signal-to-noise ratio.} The $\phi$-gradient breaks down in the sense that its signal-to-noise ratio vanishes as $\smash{\E[\widehat{\nabla}_{\phi}^\IWAE\langle \theta, \particles \rangle] / \var[\widehat{\nabla}_{\phi}^\IWAE\langle \theta, \particles \rangle]^{1/2} = \bo(K^{-1/2})}$ \citep{rainforth2018tighter}. This follows from Part~\ref{prop:properties_of_importance_sampling:2} of Proposition~\ref{prop:properties_of_importance_sampling} since $\smash{\widehat{\nabla}_{\phi}^\IWAE\langle \theta, \particles \rangle}$ constitutes a self-normalised importance-sampling approximation of $\smash{\pi_\theta(\blacktriangledown_\psi - \nabla_\phi \log q_\phi) = 0}$ (the last identity follows from Lemma~\ref{lem:tucker} with $f_\psi = w_\psi$).

  \item \textbf{Inability to achieve zero variance.}
  As pointed out in \citet{roeder2017sticking}, $\smash{\var[\widehat{\nabla}_{\phi}^\IWAE\langle \theta, \particles \rangle]] > 0}$ even in the ideal scenario that $q_\phi = \pi_\theta$ despite the fact that in this case, $\smash{w_\psi}$ is constant and hence $\smash{\var[\log \widehat{\normConst}_\theta\langle\phi, \particles\rangle] = 0}$.
 \end{itemize}
\end{remark}

%

Two modifications of $\smash{\widehat{\nabla}_{\phi}^\IWAE\langle \theta, \particles \rangle}$ have been proposed which (under \ref{as:proposal:reparametrisation}) avoid the score-function terms in \eqref{eq:iwae:phi-gradient_approximation} and hence
\begin{enumerate*}[label=(\alph*)]
 \item exhibit a stable signal-to-noise ratio as $K \to \infty$ and
 \item can achieve zero variance if $q_\phi = \pi_\theta$ (because then $\blacktriangledown_\psi \equiv 0$ since $w_\psi$ is constant).
\end{enumerate*}

\glsunset{IWAESTL}
\glsunset{IWAEDREG}

\begin{itemize}
 \item \textbf{\gls{IWAESTL}.} 
 The \glsreset{IWAESTL} \emph{\gls{IWAESTL}} gradient proposed by \citet{roeder2017sticking} heuristically ignores the score function terms (this introduces bias relative to $\smash{\widehat{\nabla}_\phi^\IWAE\langle \phi, \particles\rangle}$ whenever $K > 1$ as shown in \citet{tucker2019reparametrised}):
\begin{equation}
   \widehat{\nabla}_{\phi}^\IWAESTL\langle \theta, \particles \rangle
  \coloneqq  
 \sum_{k=1}^K \frac{w_\psi(z^k)}{\sum_{l=1}^K w_\psi(z^l)} \blacktriangledown_\psi(z^k). \label{eq:iwae-stl} 
\end{equation}
 
 \item \textbf{\gls{IWAEDREG}.} 
 The \glsreset{IWAEDREG} \emph{\gls{IWAEDREG}} gradient proposed by \citet{tucker2019reparametrised} removes the score-function terms through Lemma~\ref{lem:tucker} (i.e.\ this does not introduce bias relative to $\smash{\widehat{\nabla}_\phi^\IWAE\langle \phi, \particles\rangle}$):
\begin{equation}
 \widehat{\nabla}_{\phi}^\IWAEDREG\langle \theta, \particles \rangle
 \coloneqq \sum_{k=1}^K \biggl(\frac{w_\psi(z^k)}{\sum_{l=1}^Kw_\psi(z^l)} \biggr)^{\!\mathrlap{2}\,}\blacktriangledown_\psi(z^k). \label{eq:iwae-dreg}
\end{equation}

\end{itemize}


\glsreset{RWS}
\subsection{\Glsentryfull{RWS}}
\glsreset{RWS}

The \emph{\gls{RWS}} algorithm was proposed in \citet{bornschein2015reweighted}.\footnote{Following \citet{tucker2019reparametrised} (based on empirical results in \citet{le2019revisiting}), we only use the `wake-phase' $\phi$-updates for \gls{RWS}.} Letting $\KL(p\|q) \coloneqq \int_\spaceZ \log(p(z)/q(z)) q(z) \intDiff z$ is the \gls{KL}-divergence from $p$ to $q$, the \gls{RWS} algorithm seeks to optimise $\psi = (\theta, \phi)$ as 
\begin{gather}
\theta^\star \coloneqq \theta^\ML = \argmax_\theta \log \normConst_{\theta}, \label{eq:rws:objective:theta}\\
 \phi^\star \coloneqq \argmin_\phi \KL(\target_{\theta^\star} \| q_{\phi}).\label{eq:rws:objective:phi}
\end{gather}
The $\theta$- and $\phi$-gradients
\begin{equation}
 \begin{bmatrix}
    \nabla_\theta \log \normConst_{\theta}\\
   - \nabla_\phi \KL(\target_{\theta} \| q_{\phi})\\
 \end{bmatrix}
  = 
  \target_{\theta}
 {\begin{pmatrix}
 \nabla_\theta \log \uTarget_{\theta}\\
  \nabla_\phi \log q_{\phi}\\
 \end{pmatrix}}, \label{eq:rws:gradients_exact}
\end{equation}
are usually intractable and therefore approximated by replacing $\pi_\theta$ by the self-normalised importance sampling approximation $\smash{\hat{\target}_{\theta}\langle \phi, \particles\rangle}$ (note that this does not need \ref{as:proposal:reparametrisation}):
\begin{equation}
  \begin{bmatrix}
  \widehat{\nabla}_\theta^\RWS\langle \phi, \particles\rangle\\
  \widehat{\nabla}_\phi^\RWS\langle \theta, \particles\rangle\\
 \end{bmatrix}
 \coloneqq 
 \sum_{k=1}^{\smash{K}} \frac{w_{\psi}(z^k)}{\sum_{l=1}^K w_{\psi}(z^l)}
  {\begin{bmatrix}
 \nabla_\theta \log \uTarget_{\theta}(z^k)\\
  \nabla_\phi \log q_{\phi}(z^k)\\
 \end{bmatrix}}.\label{eq:rws:gradients_approximation}
\end{equation}
Since \eqref{eq:rws:gradients_approximation} relies on self-normalised importance sampling, it biased relative to \eqref{eq:rws:gradients_exact}. However, by Part~\ref{prop:properties_of_importance_sampling:2} of Proposition~\ref{prop:properties_of_importance_sampling} the bias of the $\theta$-gradient $\smash{
 \widehat{\nabla}_\theta^\RWS\langle \phi, \particles\rangle = \widehat{\nabla}_\theta^\IWAE\langle \phi, \particles\rangle}$ relative to $\smash{\nabla_\theta \log \normConst_\theta}$ decays as $\smash{\bo(K^{-1})}$. Appendix~\ref{app:sec:bias} discusses the impact of the bias on the $\phi$-gradients. 

The optimisation of $\theta$ and $\phi$ is carried out simultaneously. This is because (a) a better proposal $q_\phi$ reduces both bias and variance of (self-normalised) importance-sampling approximations and can therefore be leveraged for reducing the bias and variance of the $\theta$-gradients and (b) this strategy reduces the computational cost because the same set of particles $\particles$ and weights $\{w_{\psi}(z^k)\}_{k=1}^K$ is shared by both gradients. However, this simultaneous optimisation is often viewed as the main drawback of \gls{RWS} because there is no joint objective (for both $\theta$ and $\phi$).

\glsunset{RWSDREG}

\paragraph{\gls{RWSDREG}.} 

 Under \ref{as:proposal:reparametrisation}, 
\citet{tucker2019reparametrised} proposed the following \glsreset{RWSDREG} \emph{\gls{RWSDREG}} gradient which is equal to $\smash{\widehat{\nabla}_\phi^\RWS\langle \theta, \particles\rangle}$ in expectation and is derived by applying Lemma~\ref{lem:tucker} to the latter:
 \begin{equation}
  \widehat{\nabla}_\phi^\RWSDREG\langle \theta, \particles\rangle
  \coloneqq 
 \sum_{k=1}^{\smash{K}} \biggl[\frac{w_{\psi}(z^k)}{\sum_{l=1}^K w_{\psi}(z^l)} - \biggl(\frac{w_{\psi}(z^k)}{\sum_{l=1}^K w_{\psi}(z^l)}\biggr)^{\!\smash{2}}\biggr] \blacktriangledown_\psi(z^k).
 \label{eq:rws:gradient_approximation:variance_reduced:unprincipled}
 \end{equation}

\glsunset{AISLE}
\section{\glsentryshort{AISLE}: A unified adaptive importance-sampling framework}
\label{sec:aisle}

\subsection{Objective}
\label{subsec:aisle:objective}

\glsreset{AISLE}

If $\theta$ is fixed, the \gls{RWS} algorithm reduces to an adaptive importance-sampling scheme which optimises the proposal distribution by minimising the \gls{KL}-divergence from the target distribution $\target_\theta$ to the proposal $q_\phi$ \citep[see, e.g.,][]{douc2007convergence, cappe2008adaptive}. If instead $\phi$ is fixed, the \gls{RWS} algorithm reduces to a stochastic-approximation algorithm for estimating the \gls{MLE} of the generative-model parameters $\theta$. The advantage of optimising $\theta$ and $\phi$ simultaneously is that
\begin{enumerate*}[label=(\alph*)]
 \item Monte Carlo samples used to approximate the $\theta$-gradient can be re-used to approximate the $\phi$-gradient and 
 \item optimising $\phi$ typically reduces the error (both in terms of bias and variance) of the $\theta$-gradient approximation.
\end{enumerate*}

However, adapting the proposal distribution $q_\phi$ in importance-sampling schemes need not necessarily be based on minimising the \gls{KL}-divergence. Numerous other techniques exist in the literature \citep[e.g.\@][]{geweke1989bayesian, evans1991adaptive, oh1992adaptive, richard2007efficient, cornebise2008adaptive} and may sometimes be preferable. Indeed, another popular approach with strong theoretical support is based on minimising the $\chi^2$-divergence \citep[see, e.g.,][]{akyildiz2019convergence}. Based on this insight, we slightly generalise the \gls{RWS}-objective as
\begin{align}
    \theta^\star &\coloneqq \argmax_\theta \log \normConst_{\theta} (=\theta^\ML), \label{eq:aisle:objective:theta}\\
    \phi^\star &\coloneqq \argmin_\phi \Div\nolimits_{\divFun}(\target_{\theta^\star} \| q_{\phi}).\label{eq:aisle:objective:phi}
\end{align}
Here, $\Div_\divFun(p\|q) \coloneqq \int_\spaceZ \divFun(p(z)/q(z)) q(z) \intDiff z$ is some $\divFun$-divergence from $p$ to $q$. We reiterate that alternative approaches for optimising $\phi$ (which do not minimise $\divFun$-divergences) could be used. However, we state \eqref{eq:aisle:objective:phi} for concreteness as it suffices for the remainder of this work; we call the resulting algorithm \emph{\gls{AISLE}}. We stress again that \gls{AISLE} is \emph{not} introduced with the aim or claim of proposing a new algorithms but to formalise the argument that the adaptive importance-sampling paradigm avoids the drawbacks from Remark~\ref{rem:drawbacks} thus making it preferable to the multi-sample objective paradigm.



\subsection{$\theta$-gradient}
\label{subsec:aisle:gradients}
 Optimisation is again performed via a stochastic gradient-ascent. The intractable $\theta$-gradient $\smash{\nabla_\theta \log \normConst_\theta = \target_\theta(\nabla_\theta \log \uTarget_\theta)}$ is approximated as in \gls{RWS}, i.e.\ for $\smash{\particlesDistributed}$:
%
\begin{equation}
  \widehat{\nabla}_\theta^\AISLE\langle \phi, \particles\rangle
 \coloneqq \widehat{\nabla}_\theta^\RWS\langle \phi, \particles\rangle = \widehat{\nabla}_\theta^\IWAE\langle \phi, \particles\rangle. \label{eq:aisle:theta-gradient}
\end{equation}
The $\theta$-gradient is thus the same for all algorithms discussed in this work although the \gls{IWAE}-paradigm views it as an unbiased gradient for a biased objective while \gls{AISLE} (and \gls{RWS}) interpret it as a self-normalised importance-sampling (and hence biased) approximation of the gradient $\nabla_\theta \log \normConst_\theta$ for the `exact' objective.

\subsection{$\phi$-gradient special case I: \glsentryshort{RWS} and \glsentryshort{IWAESTL}}
\label{subsec:aisle:special_case_I}

The $\phi$-gradients depend on the particular choice of $\divFun$-divergence in \eqref{eq:aisle:objective:phi}. By construction, we recover \gls{RWS} as a special case of \gls{AISLE} if we define the $\divFun$-divergence through $\divFun(y) \coloneqq y \log y$ because in this case $\Div\nolimits_\divFun(p\|q) = \KL(p\|q)$ reduces to the \gls{KL}-divergence. Our main contribution in this subsection is to show that a more principled application of the identity from Lemma~\ref{lem:tucker} leads to the \gls{IWAESTL} gradient from \eqref{eq:iwae-stl}.

To derive the \gls{AISLE} $\phi$-gradients for this divergence we note that 
\begin{equation}
 - \nabla_\phi \KL(\target_{\theta} \| q_{\phi} )
  = \target_\theta(\nabla_\phi \log q_{\phi}), \label{eq:aisle:derivation_of_aisle-kl-norep}
\end{equation}
which, under \ref{as:proposal:reparametrisation}, by Lemma~\ref{lem:tucker} with $f_\psi = w_\psi$, can be written as
\begin{equation}
 \target_\theta(\nabla_\phi \log q_{\phi}) 
= q_\phi(w_\psi \nabla_\phi \log q_{\phi}) / \normConst_\theta
 = q_\phi(w_\psi \blacktriangledown_\psi) / \normConst_\theta
 = \target_\theta(\blacktriangledown_\psi). \label{eq:aisle:derivation_of_aisle-kl}
\end{equation}
We then obtain practical approximations of these gradients by plugging in $\smash{\hat{\target}_{\theta}\langle \phi, \particles\rangle}$ for $\target_\theta$.

\begin{itemize} 
 \item 
 
 \textbf{\gls{AISLEKLNOREP}\slash{}\gls{RWS}.} Without relying on any reparametrisation, \eqref{eq:aisle:derivation_of_aisle-kl-norep} yields the following gradient, which clearly equals $\smash{\widehat{\nabla}_\phi^\RWS\langle \theta, \particles\rangle}$:
 \begin{equation}
 \widehat{\nabla}_\phi^\AISLEKLNOREP\langle \theta, \particles\rangle
 \coloneqq \sum_{k=1}^{\smash{K}} \frac{w_{\psi}(z^k)}{\sum_{l=1}^K w_{\psi}(z^l)}
  \nabla_\phi \log q_\phi(z^k). \label{eq:aisle-kl-norep}
 \end{equation}

 \item 
 \textbf{\gls{AISLEKL}.} Using the reparametrisation from \ref{as:proposal:reparametrisation}, \eqref{eq:aisle:derivation_of_aisle-kl} yields the gradient:
 \begin{equation}
 \widehat{\nabla}_\phi^\AISLEKL\langle \theta, \particles\rangle
  \coloneqq \sum_{k=1}^{\smash{K}} \frac{w_{\psi}(z^k)}{\sum_{l=1}^K w_{\psi}(z^l)}
  \blacktriangledown_\psi(z^k). \label{eq:aisle-kl}
 \end{equation}
\end{itemize}


We thus arrive at the following result which demonstrates that \gls{IWAESTL} can be derived in a principled manner from \gls{AISLE}, i.e.\ without the need for a multi-sample objective.

\begin{proposition} \label{prop:aisle-kl_and_iwae-stl}
 For any $(\theta, \phi, \particles)$, $\smash{\widehat{\nabla}_\phi^\AISLEKL\langle \theta, \particles\rangle = \widehat{\nabla}_\phi^\IWAESTL\langle \theta, \particles\rangle}$. $\qedwhite$
\end{proposition}


Proposition~\ref{prop:aisle-kl_and_iwae-stl} thus provides a theoretical basis for \gls{IWAESTL} which was previously viewed as an alternative gradient for \gls{IWAE} for which it is biased and only heuristically justified. Furthermore, the fact that \gls{IWAESTL} exhibited good empirical performance in \citet{tucker2019reparametrised} even in an example in which \gls{RWS} broke down, suggests that this breakdown may not be due to \gls{RWS}' lack of optimising a joint objective as previously conjectured.

Finally, recall that \citet{tucker2019reparametrised} obtained an alternative `doubly-reparametrised' \gls{RWS} $\phi$-gradient $\smash{\widehat{\nabla}_\phi^\RWSDREG\langle \theta, \particles\rangle}$ given in \eqref{eq:rws:gradient_approximation:variance_reduced:unprincipled} by \emph{first} replacing the exact (but intractable) $\phi$-gradient from \eqref{eq:aisle:derivation_of_aisle-kl-norep} by the self-normalised importance-sampling approximation $\smash{\widehat{\nabla}_\phi^\RWS\langle \theta, \particles\rangle}$ and \emph{then} applying the identity from Lemma~\ref{lem:tucker}. Note that this may result in a variance reduction but does not change the bias of the gradient estimator. In contrast, \gls{AISLEKL} is derived by \emph{first} applying Lemma~\ref{lem:tucker} to the exact (\gls{RWS}) $\phi$-gradient and \emph{then} approximating the resulting expression. 
This can potentially reduce both bias and variance.
%

\subsection{$\phi$-gradient special case II: \glsentryshort{IWAEDREG}}
\label{subsec:aisle:special_case_II}
We now demonstrate that the \gls{IWAEDREG} gradient can be recovered as a special case of \gls{AISLE} (up to a proportionality constant). To establish this relationship, we take $\divFun(y) \coloneqq (y-1)^2$ so that 
\begin{equation}
 \Div\nolimits_\divFun(p\|q) = \chisq(p\|q) 
 \coloneqq \int_\spaceZ \biggl(\frac{p(z)}{q(z)} - 1\biggr)^{\!\mathrlap{2}} \, q(z) \intDiff z 
 = \int_\spaceZ \frac{p(z)}{q(z)} p(z) \intDiff z - 1,
\end{equation}
is the $\smash{\chi^2}$-divergence. Minimising this divergence is natural in importance sampling since $\smash{\chisq(\target_\theta\|q_\phi) = \var_{z \sim q_\phi}[w_\psi/\normConst_\theta]}$ is the variance of the importance weights. 

To derive the \gls{AISLE} $\phi$-gradients for this divergence we note that 
\begin{equation}
 - \nabla_\phi \chisq(\target_{\theta} \| q_{\phi} )
 = - \target_\theta(\nabla_\phi w_\psi) / \normConst_\theta
 = \target_\theta(w_\psi \nabla_\phi \log q_\phi) / \normConst_\theta, \label{eq:aisle:derivation_of_aisle-chisq-norep}
\end{equation}
which, under \ref{as:proposal:reparametrisation}, by Lemma~\ref{lem:tucker} with $f_\psi = w_\psi^2$, can be written as
\begin{align}
 \target_\theta(w_\psi \nabla_\phi \log q_\phi) / \normConst_\theta
 &= q_\phi(w_\psi^2 \nabla_\phi \log q_\phi) /  \normConst_\theta^2 \\
 & = q_\phi(w_\psi^2 \nabla_\phi [{\log} \circ w_{\psi^{\mathrlap{\prime}}}^2 \circ h_\phi]|_{\psi^\prime = \psi} \circ {h_\phi^{-1}}) /  \normConst_\theta^2\\
 & = \target_\theta(2 w_\psi \blacktriangledown_\psi) / \normConst_\theta.
 \label{eq:aisle:derivation_of_aisle-chisq}
\end{align}
Again plugging in $\smash{\hat{\target}_{\theta}\langle \phi, \particles\rangle}$ for $\target_\theta$ and $\smash{\widehat{\normConst}_\theta\langle \phi, \particles\rangle}$ for $\normConst_\theta$ yields the following approximations.

\begin{itemize} 
 \item \textbf{\gls{AISLECHISQNOREP}.} Without relying on any reparametrisation, \eqref{eq:aisle:derivation_of_aisle-chisq-norep} yields the following gradient which is also proportional to the `score gradient' from \citet[][Appendix~G]{dieng2017variational}:
 \begin{equation}
 \widehat{\nabla}_\phi^\AISLECHISQNOREP\langle \theta, \particles\rangle
  \coloneqq K \sum_{k=1}^{\smash{K}} \biggl(\frac{w_{\psi}(z^k)}{\sum_{l=1}^K w_{\psi}(z^l)}\biggr)^{\mathrlap{2}}
  \nabla_\phi \log q_\phi(z^k).\label{eq:aisle-chisq-norep}
 \end{equation}
 \item \textbf{\gls{AISLECHISQ}.} Using the reparametrisation from \ref{as:proposal:reparametrisation}, \eqref{eq:aisle:derivation_of_aisle-chisq} yields the gradient:
 \begin{equation}
 \widehat{\nabla}_\phi^\AISLECHISQ\langle \theta, \particles\rangle
 \coloneqq 2 K \sum_{k=1}^{\smash{K}} \biggl(\frac{w_{\psi}(z^k)}{\sum_{l=1}^K w_{\psi}(z^l)}\biggr)^{\mathrlap{2}}
  \blacktriangledown_\psi(z^k). \label{eq:aisle-chisq}
 \end{equation}
\end{itemize}

We thus arrive at the following result which demonstrates that \gls{IWAEDREG} 
can be derived (up to the proportionality factor $2K$) in a principled manner from \gls{AISLE}, i.e.\ without the need for a multi-sample objective.

\begin{proposition} \label{prop:aisle-chisq_and_iwae-dreg}
 For any $(\theta, \phi, \particles)$, $\smash{\widehat{\nabla}_\phi^\AISLECHISQ\langle \theta, \particles\rangle = 2 K \widehat{\nabla}_\phi^\IWAEDREG\langle \theta, \particles\rangle}$. $\qedwhite$
\end{proposition}

Note that if the implementation normalises the gradients, e.g.\ as effectively done by \gls{ADAM} \citep{kingma2015adam}, the constant factor cancels out and \gls{AISLECHISQ} becomes equivalent to \gls{IWAEDREG}. Otherwise (e.g.\ in plain stochastic gradient-ascent) Proposition~\ref{prop:aisle-chisq_and_iwae-dreg} shows that the learning rate needs to be scaled as $\bo(K)$ for the \gls{IWAE} or \gls{IWAEDREG} $\phi$-gradients. 

\section{Conclusion}
\label{sec:conclusion}
\glsreset{AISLE}
\glsreset{IWAE}
\glsreset{RWS}
\glsreset{IWAESTL}
\glsreset{IWAEDREG}

We have shown that the adaptive-importance sampling paradigm of the \emph{\gls{RWS}} \citep{bornschein2015reweighted} is preferable to the multi-sample objective paradigm of \emph{\glspl{IWAE}} \citep{burda2016importance} because the former achieves all the goals of the latter whilst avoiding its drawbacks. To formalise this argument, we have introduced a simple, unified adaptive-importance-sampling framework termed \emph{\gls{AISLE}} (which slightly generalises the \gls{RWS} algorithm) and have proved that \gls{AISLE} allows us to derive the \emph{\gls{IWAESTL}} gradient from \citet{roeder2017sticking} and the \emph{\gls{IWAEDREG}} gradient from \citet{tucker2019reparametrised} as special cases.



We hope that this work highlights the potential for further improving variational techniques by drawing upon the vast body of research on (adaptive) importance sampling in the computational statistics literature. Conversely, the methodological connections established in this work may also serve to emphasise the utility of the reparametrisation trick from \citet{kingma2014auto, tucker2019reparametrised} to computational statisticians. 


In a companion article \citep{finke2019variational}, we extend the present work to the \emph{\glsdesc{VSMC}} methods from \citet{maddison2017filtering, le2018auto, naeseth2018variational} and to the \emph{tensor Monte Carlo} approach from \citet{aitchison2018tensor}.

%

\renewcommand*{\bibfont}{\footnotesize}
\setlength{\bibsep}{3pt plus 0.3ex}
\bibliography{sub.bbl}

\begin{thebibliography}{}

\bibitem[{Aitchison}, 2018]{aitchison2018tensor}
{Aitchison}, L. (2018).
\newblock {Tensor Monte Carlo: particle methods for the GPU era}.
\newblock {\em arXiv e-prints}, 1806.08593.

\bibitem[Andrieu et~al., 2010]{andrieu2010particle}
Andrieu, C., Doucet, A., and Holenstein, R. (2010).
\newblock Particle {M}arkov chain {M}onte {C}arlo methods.
\newblock {\em Journal of the Royal Statistical Society: Series B (Statistical
  Methodology)}, 72(3):269--342.
\newblock With discussion.

\bibitem[Andrieu and Roberts, 2009]{andrieu2009pseudo}
Andrieu, C. and Roberts, G.~O. (2009).
\newblock The pseudo-marginal approach for efficient {M}onte {C}arlo
  computations.
\newblock {\em The Annals of Statistics}, 37(2):697--725.

\bibitem[Bamler et~al., 2017]{bamler2017perturbative}
Bamler, R., Zhang, C., Opper, M., and Mandt, S. (2017).
\newblock Perturbative black box variational inference.
\newblock {\em Advances in {N}eural {I}nformation {P}rocessing {S}ystems
  ({NeurIPS})}, pages 5079--5088.

\bibitem[Bornschein and Bengio, 2015]{bornschein2015reweighted}
Bornschein, J. and Bengio, Y. (2015).
\newblock Reweighted wake-sleep.
\newblock In {\em 3rd {I}nternational {C}onference on {L}earning
  {R}epresentations ({ICLR})}.

\bibitem[Burda et~al., 2016]{burda2016importance}
Burda, Y., Grosse, R., and Salakhutdinov, R. (2016).
\newblock Importance weighted autoencoders.
\newblock In {\em 4th {I}nternational {C}onference on {L}earning
  {R}epresentations ({ICLR})}.

\bibitem[Capp{\'e} et~al., 2008]{cappe2008adaptive}
Capp{\'e}, O., Douc, R., Guillin, A., Marin, J.-M., and Robert, C.~P. (2008).
\newblock Adaptive importance sampling in general mixture classes.
\newblock {\em Statistics and Computing}, 18(4):447--459.

\bibitem[Cornebise et~al., 2008]{cornebise2008adaptive}
Cornebise, J., Moulines, {\'E}., and Olsson, J. (2008).
\newblock Adaptive methods for sequential importance sampling with application
  to state space models.
\newblock {\em Statistics and Computing}, 18(4):461--480.

\bibitem[Cremer et~al., 2017]{cremer2017reinterpreting}
Cremer, C., Morris, Q., and Duvenaud, D. (2017).
\newblock Reinterpreting importance-weighted autoencoders.
\newblock In {\em 5th {I}nternational {C}onference on {L}earning
  {R}epresentations ({ICLR})}.

\bibitem[{Deniz Akyildiz} and {M{\'\i}guez}, 2019]{akyildiz2019convergence}
{Deniz Akyildiz}, {\"O}. and {M{\'\i}guez}, J. (2019).
\newblock Convergence rates for optimised adaptive importance samplers.
\newblock {\em arXiv e-prints}, 1903.12044.

\bibitem[Dieng et~al., 2017]{dieng2017variational}
Dieng, A.~B., Tran, D., Ranganath, R., Paisley, J., and Blei, D. (2017).
\newblock Variational inference via $\chi$ upper bound minimization.
\newblock {\em Advances in Neural Information Processing Systems (NeurIPS)},
  pages 2732--2741.

\bibitem[Domke and Sheldon, 2018]{domke2018importance}
Domke, J. and Sheldon, D.~R. (2018).
\newblock Importance weighting and variational inference.
\newblock {\em Advances in Neural Information Processing Systems (NeurIPS)},
  pages 4475--4484.

\bibitem[Douc et~al., 2007]{douc2007convergence}
Douc, R., Guillin, A., Marin, J.-M., and Robert, C.~P. (2007).
\newblock Convergence of adaptive mixtures of importance sampling schemes.
\newblock {\em The Annals of Statistics}, 35(1):420--448.

\bibitem[Evans, 1991]{evans1991adaptive}
Evans, M. (1991).
\newblock Adaptive importance sampling and chaining.
\newblock {\em Statistical Numerical Integration, Contemporary Mathematics},
  115:137--143.

\bibitem[Finke, 2015]{finke2015extended}
Finke, A. (2015).
\newblock {\em On extended state-space constructions for {M}onte {C}arlo
  methods}.
\newblock PhD thesis, Department of Statistics, University of Warwick, UK.

\bibitem[{Finke} and {Thi\'ery}, 2019]{finke2019variational}
{Finke}, A. and {Thi\'ery}, A.~H. (2019).
\newblock On variational sequential {M}onte {C}arlo methods.
\newblock Manuscript in preparation.

\bibitem[Geweke, 1989]{geweke1989bayesian}
Geweke, J. (1989).
\newblock Bayesian inference in econometric models using {M}onte {C}arlo
  integration.
\newblock {\em Econometrica}, 57(6):1317--1339.

\bibitem[Ionides, 2008]{ionides2008truncated}
Ionides, E.~L. (2008).
\newblock Truncated importance sampling.
\newblock {\em Journal of Computational and Graphical Statistics},
  17(2):295--311.

\bibitem[Kingma and Ba, 2015]{kingma2015adam}
Kingma, D.~P. and Ba, J.~L. (2015).
\newblock {ADAM}: {A} method for stochastic optimization.
\newblock In {\em 3rd {I}nternational {C}onference on {L}earning
  {R}epresentations ({ICLR})}.

\bibitem[Kingma and Welling, 2014]{kingma2014auto}
Kingma, D.~P. and Welling, M. (2014).
\newblock Auto-encoding variational {B}ayes.
\newblock In {\em 2nd {I}nternational {C}onference on {L}earning
  {R}epresentations ({ICLR})}.

\bibitem[Kong et~al., 1994]{kong1994sequential}
Kong, A., Liu, J.~S., and Wong, W.~H. (1994).
\newblock Sequential imputations and {B}ayesian missing data problems.
\newblock {\em Journal of the American Statistical Association},
  89(425):278--288.

\bibitem[Le et~al., 2018]{le2018auto}
Le, T.~A., Igl, M., Rainforth, T., Jin, T., and Wood, F. (2018).
\newblock Auto-encoding sequential {M}onte {C}arlo.
\newblock In {\em 6th {I}nternational {C}onference on {L}earning
  {R}epresentations ({ICLR})}.

\bibitem[Le et~al., 2019]{le2019revisiting}
Le, T.~A., Kosiorek, A.~R., Siddharth, N., Teh, Y.~W., and Wood, F. (2019).
\newblock Revisiting reweighted wake-sleep for models with stochastic control
  flow.
\newblock In {\em Proceedings of the 35th Conference on Uncertainty in
  Artificial Intelligence (UAI)}.

\bibitem[Lee, 2011]{lee2011auxiliary}
Lee, A. (2011).
\newblock {\em On auxiliary variables and many-core architectures in
  computational statistics}.
\newblock PhD thesis, Department of Statistics, University of Oxford, UK.

\bibitem[Liu, 1996]{liu1996metropolized}
Liu, J.~S. (1996).
\newblock Metropolized independent sampling with comparisons to rejection
  sampling and importance sampling.
\newblock {\em Statistics and Computing}, 6(2):113--119.

\bibitem[Liu, 2001]{liu2001monte}
Liu, J.~S. (2001).
\newblock {\em {M}onte {C}arlo Strategies in Scientific Computing}.
\newblock Springer Series in Statistics. Springer.

\bibitem[Maddison et~al., 2017]{maddison2017filtering}
Maddison, C.~J., Lawson, J., Tucker, G., Heess, N., Norouzi, M., Mnih, A.,
  Doucet, A., and Teh, Y.~W. (2017).
\newblock Filtering variational objectives.
\newblock {\em Advances in Neural Information Processing Systems (NeurIPS)},
  pages 6573--6583.

\bibitem[Naesseth et~al., 2018]{naeseth2018variational}
Naesseth, C.~A., Linderman, S.~W., Ranganath, R., and Blei, D.~M. (2018).
\newblock Variational sequential {M}onte {C}arlo.
\newblock In {\em 21st {I}nternational {C}onference on {A}rtificial
  {I}ntelligence and {S}tatistics (AISTATS)}.

\bibitem[Oh and Berger, 1992]{oh1992adaptive}
Oh, M.-S. and Berger, J.~O. (1992).
\newblock Adaptive importance sampling in {M}onte {C}arlo integration.
\newblock {\em Journal of Statistical Computation and Simulation},
  41(3-4):143--168.

\bibitem[Paisley et~al., 2012]{paisley2012variational}
Paisley, J., Blei, D., and Jordan, M. (2012).
\newblock Variational {B}ayesian inference with stochastic search.
\newblock In {\em 29th {I}nternational {C}onference on {M}achine {L}earning
  ({ICML})}.

\bibitem[Rainforth et~al., 2018]{rainforth2018tighter}
Rainforth, T., Kosiorek, A.~R., Le, T.~A., Maddison, C.~J., Igl, M., Wood, F.,
  and Teh, Y.~W. (2018).
\newblock Tighter variational bounds are not necessarily better.
\newblock In {\em {B}ayesian {D}eep {L}earning ({NeurIPS} 2018 workshop)}.

\bibitem[Richard and Zhang, 2007]{richard2007efficient}
Richard, J.-F. and Zhang, W. (2007).
\newblock Efficient high-dimensional importance sampling.
\newblock {\em Journal of Econometrics}, 141(2):1385--1411.

\bibitem[Roeder et~al., 2017]{roeder2017sticking}
Roeder, G., Wu, Y., and Duvenaud, D.~K. (2017).
\newblock Sticking the landing: {S}imple, lower-variance gradient estimators
  for variational inference.
\newblock {\em Advances in Neural Information Processing Systems (NeurIPS)},
  pages 6925--6934.

\bibitem[Tucker et~al., 2019]{tucker2019reparametrised}
Tucker, G., Lawson, D., Gu, S., and Maddison, C.~J. (2019).
\newblock Doubly reparameterized gradient estimators for {M}onte {C}arlo
  objectives.
\newblock In {\em 7th {I}nternational {C}onference on {L}earning
  {R}epresentations ({ICLR})}.

\bibitem[Xu et~al., 2019]{xu2019variance}
Xu, M., Quiroz, M., Kohn, R., and Sisson, S.~A. (2019).
\newblock Variance reduction properties of the reparameterization trick.
\newblock In {\em The 22nd International Conference on Artificial Intelligence
  and Statistics (AISTATS)}, pages 2711--2720.

\end{thebibliography}

\appendix

\section{On the r\^ole of the self-normalisation bias within \gls{RWS}\slash{}\gls{AISLE}}
\label{app:sec:bias}

\subsection{The self-normalisation bias}
\label{app:subsec:self-normalisation_bias}


Within the self-normalised importance-sampling approximation, the number of particles, $K$, interpolates between two extremes:
\begin{itemize}
 \item As $K \uparrow \infty$, $\smash{\hat{\pi}_\theta\langle \phi, \mathbf{z}\rangle}(f)$ becomes an increasingly accurate approximation of $\pi_\theta(f)$.
 \item For $K=1$, however, $\smash{\hat{\pi}_\theta\langle \phi, \mathbf{z}\rangle}(f) = f(z^1)$ reduces to a vanilla Monte Carlo approximation of $q_\phi(f)$ (because the single self-normalised importance weight is always equal to $1$).
\end{itemize}
This leads to the following insight about the estimators $\smash{\widehat{\nabla}_\phi^\AISLEKL\langle \theta, \mathbf{z}\rangle}$ and $\smash{\widehat{\nabla}_\phi^\AISLECHISQ\langle \theta, \mathbf{z}\rangle}$.
\begin{itemize}
 \item As $K \uparrow \infty$, these two estimators become increasingly accurate approximations of the \emph{`inclusive'}-divergence gradients $\smash{-\nabla_\phi \KL(\pi_\theta\|q_\phi) = \pi_\theta(\blacktriangledown_\phi)}$ and $\smash{-\nabla_\phi \chisq(\pi_\theta\|q_\phi) = 2\pi_\theta([w_\psi/\normConst_\theta]\blacktriangledown_\phi)}$, respectively.
 
 \item For $K=1$, however, these two estimators reduce to vanilla Monte Carlo approximations of the \emph{`exclusive'}-divergence gradients $\smash{-\nabla_\phi \KL(q_\phi\|\pi_\theta) = q_\phi(\blacktriangledown_\phi)}$ and $\smash{-2 \nabla_\phi \KL(q_\phi\|\pi_\theta) = 2 q_\phi(\blacktriangledown_\phi)}$, respectively.
\end{itemize}
This is similar to the standard \gls{IWAE} $\phi$-gradient which also represents a vanilla Monte Carlo approximation of $\smash{-\nabla_\phi \KL(q_\phi\|\pi_\theta)}$ if $K = 1$ as \gls{IWAE} reduces to a \gls{VAE} in this case.

Characterising the small-$K$ self-normalisation bias of the reparametrisation-free \gls{AISLE} $\phi$ gradients, \gls{AISLEKLNOREP} and \gls{AISLECHISQNOREP}, is more difficult because if $K = 1$, they constitute vanilla Monte Carlo approximations of $\smash{q_\phi(\nabla_\phi \log q_\phi) = 0}$. Nonetheless, \citet[][Figure~5]{le2019revisiting} lends some support to the hypothesis that the small-$K$ self-normalisation bias of these gradients also favours a minimisation of the exclusive \gls{KL}-divergence.

\subsection{Inclusive vs exclusive \gls{KL}-divergence minimisation}
\label{app:subsec:inclusive_vs_exclusive}

Recall that the main motivation for use of \glspl{IWAE} (instead of \glspl{VAE}) was the idea that we could use self-normalised importance-sampling approximations with $K > 1$ particles to reduce the bias of the $\theta$-gradient relative to $\smash{\nabla_\theta \log \normConst_\theta}$. The error of such (self-normalised) importance-sampling approximations can be controlled by ensuring that $q_\phi$ is close to $\pi_\theta$ (in some suitable sense) in any part of the space $\spaceZ$ in which $\pi_\theta$ has positive probability mass. For instance, it is well known that the error will be small if the `inclusive' \gls{KL}-divergence $\smash{\KL(\pi_\theta \| q_\phi)}$ is small as this implies well-behaved importance weights. In contrast, a small `exclusive' \gls{KL}-divergence $\smash{\KL(q_\phi \| \pi_\theta)}$ is not sufficient for well-behaved importance weights because the latter only ensures that $q_\phi$ is close to $\pi_\theta$ in those parts of the space $\spaceZ$ in which $q_\phi$ has positive probability mass.

Let $\mathcal{Q} \coloneqq \{q_\phi\}$ (which is indexed by $\phi$) be the family of proposal distributions\slash{}the variational family. Then we can distinguish two scenarios.
\begin{enumerate}
 \item \label{enum:sufficiently_expressive_variational_family}

\textbf{Sufficiently expressive $\mathcal{Q}$.} For the moment, assume that the family $\mathcal{Q}$ is flexible (`expressive') enough in the sense that it contains a distribution $q_{\phi^\star}$ which is (at least approximately) equal to $\pi_\theta$ and that our optimiser can reach the value $\phi^\star$ of $\phi$. In this case, minimising the exclusive \gls{KL}-divergence can still yield well-behaved importance weights because in this case, $\smash{\phi^\star \coloneqq \argmin_\phi \KL(\pi_\theta \| q_\phi)}$ is (at least approximately) equal to $\smash{\argmin_\phi \KL(q_\phi \| \pi_\theta)}$. 

\item \label{enum:insufficiently_expressive_variational_family} \textbf{Insufficiently expressive $\mathcal{Q}$.} In general, the family $\mathcal{Q}$ is not flexible enough in the sense that all of its members are `far away' from $\pi_\theta$, e.g.\ if the $D$ components $z_1, \dotsc, z_D$ of $z = z_{1:D}$ are highly correlated under $\pi_\theta$ whilst $\smash{q_\phi(z) = \prod_{d=1}^D q_{\phi,d}(z_d)}$ is fully factorised. In this case, minimising the exclusive \gls{KL}-divergence could lead to poorly-behaved importance weights and we should optimise $\smash{\phi^\star \coloneqq \argmin_\phi \KL(\pi_\theta \| q_\phi)}$ as discussed above.

\end{enumerate}

\begin{remark} \label{rem:minimising_exclusive_divergence_preferable}
 In Scenario~\ref{enum:sufficiently_expressive_variational_family} above, i.e.\ for a sufficiently flexible $\mathcal{Q}$, using a gradient-descent algorithm which seeks to minimise the \emph{exclusive} divergence can sometimes be preferable to a gradient-descent algorithm which seeks to minimise the \emph{inclusive} divergence. This is because both find (approximately) the same optimum but the latter may exhibit faster convergence in some applications. In such scenarios, the discussion in Subsection~\ref{app:subsec:self-normalisation_bias} indicates that a smaller number of particles, $K$, could then be preferable for some of the $\phi$-gradients because 
 \begin{enumerate*}[label=(\alph*)]
  \item the $\smash{\bo(K^{-1})}$ self-normalisation bias outweighs the $\smash{\bo(K^{-1/2})}$ standard deviation and 
  \item the direction of this bias may favour faster convergence.
 \end{enumerate*}
\end{remark}

Unfortunately, simply setting $K=1$ for the approximation of the $\phi$-gradients\footnote{Within the \gls{IWAE}-paradigm, using different numbers of particles for the $\theta$ and $\phi$-gradients was recently proposed in \citet{rainforth2018tighter, le2018auto} who termed this approach \emph{`alternating \glsdescplural{ELBO}'}, albeit their aim was to circumvent the signal-to-noise ratio breakdown of the \gls{IWAE} $\phi$-gradient which is distinct from the phenomenon discussed here.} is not necessarily optimal because
\begin{itemize}
 \item even in the somewhat idealised scenario \ref{enum:sufficiently_expressive_variational_family} above and even if the direction of the self-normalisation bias encourages faster convergence, increasing $K$ is still desirable to reduce the variance of the gradient approximations;
 
 \item not using the information contained in \emph{all} $K$ particles and weights (which have already been sampled\slash{}calculated to approximate the $\theta$-gradient) seems wasteful;
 
 \item if $K= 1$, the reparametrisation-free \gls{AISLE}$\phi$ gradients, \gls{AISLEKLNOREP} and \gls{AISLECHISQNOREP} are simply vanilla Monte Carlo estimates of $0$ and the \gls{RWSDREG} $\phi$-gradient is then equal to $0$.
\end{itemize}

\subsection{Regularisation}
\label{subsec:bias-variance_trade-offs}

We propose here to `regularise' the importance weights. That is, letting
\begin{align}
 \mathit{ESS}(w_{1:K}) \coloneqq \biggl[\sum_{k=1}^K \biggl( \frac{w_k}{\sum_{l=1}^K w_l}\biggr)^{\!\!2\,}\biggr]^{-1} \in [1,K],
\end{align}
 be the \emph{effective sample size} \citep{kong1994sequential, liu1996metropolized}, we propose to replace the weights $\smash{w_{\psi,x}(z^k) \eqqcolon w_k}$ in any of the gradients discussed in this work by $\smash{w_k^{\mathrlap{\alpha^\star}}}$, where $\smash{\alpha^\star \coloneqq \sup\{\alpha \in [0,1] : \mathit{ESS}(w_{1:K}^\alpha) \geq \eta K\}}$ can be inexpensively found with bisection; the tuning parameter $\eta \in (0,1]$ governs the amount of regularisation.

This weight-regularisation strategy typically reduces the variance but increases the bias of (self-normalised) importance-sampling approximations $\smash{\hat{\pi}_\theta\langle \phi, \mathbf{z}\rangle(f)}$ relative to $\pi_\theta(f)$. When applied to the $\phi$-gradients, it may be interpreted in two different ways.
\begin{enumerate}
 \item Within Scenario~\ref{enum:sufficiently_expressive_variational_family}, we can view such a regularisation strategy as a means of interpolating between $K$-sample vanilla Monte Carlo approximations of gradients of the exclusive \gls{KL}-divergence ($\alpha^\star = 0$) and self-normalised importance-sampling approximations of inclusive-divergence type gradients ($\alpha^\star = 1$).
 
 \item Within Scenario~\ref{enum:insufficiently_expressive_variational_family}, we can view such a regularisation strategy as a means reducing the overall approximation error of importance sampling through a more favourable bias--variance trade-off \citep{ionides2008truncated}. Thus, in this scenario, we interpret the regularisation as a way of attaining $\phi$-gradients whose overall error (relative to the intractable inclusive-divergence gradient) is reduced. 
\end{enumerate}
We note that the second interpretation applies to the $\theta$-gradient in either scenario. Furthermore, such weight regularisation also circumvents the signal-to-noise ratio breakdown in the standard \gls{IWAE} $\phi$-gradient. For an alternative weight-regularisation strategy used in the context of variational inference, see \citet{bamler2017perturbative}.

Finally, a more principled approach would be to regularise the problem rather than the approximation. That is, we could instead regularise the generative model, i.e.\ replace $\target_{\theta,x}(z)$ by a regularised distribution, e.g.\ by a distribution proportional to $\smash{\target_{\vartheta,x}(z)^{\alpha} q_{\varphi,x}(z)^{1-\alpha}}$ for $\alpha \in (0,1]$ (and also replacing $\theta$ by $(\vartheta, \varphi)$). We are currently investigating such ideas.

\section{Empirical illustration}
\label{app:sec:illustrations}

\subsection{Algorithms}
\label{app:subsec:algorithms}

\glsunset{RWSDREG}
\glsunset{AISLEKLNOREP}
\glsunset{AISLEKL}
\glsunset{AISLECHISQNOREP}
\glsunset{AISLECHISQ}

In these supplementary materials, we illustrate the different $\phi$-gradient estimators (recall that all algorithms discussed in this work share the same $\theta$-gradient estimator). Specifically, we compare the following approximations.
\begin{itemize}
 \item \textbf{\gls{AISLEKLNOREP}.} \label{enum:aisle-kl-norep}
 The gradient for \gls{AISLE} based on the \gls{KL}-divergence without any further reparametrisation from \eqref{eq:aisle-kl-norep} i.e.\ this coincides with the standard \gls{RWS}-gradient from \eqref{eq:rws:gradients_approximation}. This gradient does not require \ref{as:proposal:reparametrisation} but does not achieve zero variance even if $q_\phi = \pi_\theta$.
 

 \item\label{enum:aisle-kl} \textbf{\gls{AISLEKL}.} 
  The gradient for \gls{AISLE} based on the \gls{KL}-divergence after reparametrising and exploiting the identity from Lemma~\ref{lem:tucker}; it is given by \eqref{eq:aisle-kl} and coincides with the \gls{IWAESTL}-gradient from \eqref{eq:iwae-stl}.
  
 \item \label{enum:aisle-chisq-norep} \textbf{\gls{AISLECHISQNOREP}.} 
  The gradient for \gls{AISLE} based on the $\chi^2$-divergence without any reparametrisation given in \eqref{eq:aisle-chisq-norep}. This gradient again does not require \ref{as:proposal:reparametrisation} but does not achieve zero variance even if $q_\phi = \pi_\theta$.
 \item \label{enum:aisle-chisq} \textbf{\gls{AISLECHISQ}.}  
 The gradient for \gls{AISLE} based on the $\chi^2$-divergence after reparametrising and exploiting the identity from Lemma~\ref{lem:tucker}; it is given by \eqref{eq:aisle-chisq} and is alsow proportional to \gls{IWAEDREG} from \citet{tucker2019reparametrised} which was stated in \eqref{eq:iwae-dreg}. When normalising the gradients (as, e.g.\ implicitly done by optimisers such as \gls{ADAM} \citealp{kingma2015adam}) the proportionality constant cancels out so that both these gradient approximations lead to computationally the same algorithm.
 
 
 \item \label{enum:iwae} \textbf{\gls{IWAE}.} The gradient for \gls{IWAE} employing the reparametrisation trick from \citet{kingma2014auto}. Its sampling approximation is given in \eqref{eq:iwae:phi-gradient_approximation}. Recall that this is the $\phi$-gradient whose signal-to-noise ratio degenerates with $K$ as pointed out in \citet{rainforth2018tighter} (and which also cannot achieve zero variance even if $q_\phi = \pi_\theta$).
 
  \item \label{enum:iwae-dreg} \textbf{\gls{IWAEDREG}.}  
  The `doubly-reparametrised' \gls{IWAE} gradient from  \eqref{eq:iwae-dreg} which was proposed in \citet{tucker2019reparametrised}. It is proportional to \gls{AISLECHISQ}.
 
 \item \label{enum:rws-dreg} \textbf{\gls{RWSDREG}.} The `doubly-reparametrised' \gls{RWS} $\phi$-gradient from \eqref{eq:rws:gradient_approximation:variance_reduced:unprincipled} which was proposed in \citet{tucker2019reparametrised} who derived it by applying the identity from Lemma~\ref{lem:tucker} to the \gls{RWS} $\phi$-gradient.
 

\end{itemize}

\subsection{Model}

\paragraph{Generative model.} 

We have $N$ $D$-dimensional observations $\smash{x^{(1)}, \dotsc, x^{(N)} \in \reals^{D}}$ and $N$ $D$-dimensional latent variables $\smash{z^{(1)}, \dotsc, z^{(N)} \in \reals^{D}}$. Unless otherwise stated, any vector $y \in \reals^D$ is to be viewed as a $D \times 1$ column vector. 

Hereafter, wherever necessary, we add an additional subscript to make the dependence on the observations explicit. The joint law (the `generative model'), parametrised by $\theta$, of the observations and latent variables then factorises as
\begin{align}
  \prod_{n=1}^N p_\theta(z^{(n)}) p_\theta(x^{(n)}|z^{(n)}) = \prod_{n=1}^N \uTarget_{\theta,x^{(n)}}(z^{(n)}).
\end{align}
We model each latent variable--observation pair $(z, x)$ as
\begin{align}
 p_\theta(z) & \coloneqq \dN(z; \mu, \varSigma),\\
 p_\theta(x|z) & \coloneqq \dN(x; z; \iMat),
\end{align}
where $\smash{\theta \coloneqq \mu = \mu_{1:D} \in \reals^{D}}$, where $\smash{\varSigma \coloneqq (\sigma_{d,d'})_{(d,d') \in \{1,\dotsc,D\}} \in \reals^{D \times D}}$ is assumed to be known and where $\iMat$ denotes the $D\times D$-identity matrix. For any $\theta$, 
\begin{align}
 \normConst_{\theta, x} & = p_{\theta}(x) =  \dN(x; \mu, \iMat + \varSigma), \label{eq:toy_gaussian_normalising_constant}\\
\target_{\theta,x}(z) & = p_{\theta}(z|x) =  \dN(z; \nu_{\theta,x}, P),
 \label{eq:ex:toy_gaussian_posterior}
\end{align}
with $\smash{P \coloneqq (\varSigma^{-1} + \iMat)^{-1}}$ and $\smash{\nu_{\theta,x} \coloneqq P(\varSigma^{-1}\mu + x)}$. In particular, \eqref{eq:toy_gaussian_normalising_constant} implies that $\smash{\theta^\ML = \frac{1}{N} \sum_{n = 1}^N x^{(n)}}$.

\paragraph{Proposal\slash{}variational approximation.} We take the proposal distributions as a fully-factored Gaussian:
\begin{align}
 q_{\phi,x}(z) \coloneqq \dN(z; A x + b, C), \label{eq:ex:toy_gaussian_proposal}
\end{align}
where $A = (a_{d,d'})_{(d,d') \in \{1,\dotsc,D\}^2} \in \reals^{D \times D}$, $b = b_{1:D} \in \reals^{D}$ and, for  $c_{1:D} \eqqcolon c \in \reals^{D}$, $C \coloneqq \diag(\eul^{2c_1},\dotsc,\eul^{2c_D})$. The parameters to optimise are thus
\begin{align}
 \phi \coloneqq (a_{1}^\T, \dotsc, a_{D}^\T, b^\T, c^\T),
\end{align}
where $\smash{a_d \coloneqq [a_{d,1}, a_{d,2}, \dotsc, a_{d,D}]^\T \in \reals^{D\times1}}$ denotes the column vector formed by the elements in the $d$th row of $A$. Furthermore, for the reparametrisation trick, we take $\smash{q(\particleReparametrised) \coloneqq \dN(\particleReparametrised; 0, \iMat)}$, where $\smash{0 \in \reals^{D}}$ is a vector whose elements are all $0$, so that
\begin{align}
 h_{\phi,x}(\particleReparametrised) & \coloneqq Ax + b + C^{1/2} \particleReparametrised,
\end{align}
which means that $\smash{h_{\phi,x}^{-1}(z) = C^{-1/2}(z - Ax - b)}$.

Note that the mean of the proposal in \eqref{eq:ex:toy_gaussian_proposal} coincides with the mean of the posterior in \eqref{eq:ex:toy_gaussian_posterior} if $A = P$ and $b = P \varSigma^{-1} \mu$.

This model is similar to the one used as a benchmark in \citet[][Section~4]{rainforth2018tighter} and also in \citet[][Section~6.1]{tucker2019reparametrised} who specified both the generative model and the variational approximation to be isotropic Gaussians. Specifically, their setting can be recovered by taking $\varSigma \coloneqq \iMat$ and fixing $c_d = \log(2/3)/2$ so that $C = \frac{2}{3} \iMat$ throughout. Here, in order to investigate a slightly more realistic scenario, we also allow for the components of the latent vectors $z$ to be \emph{correlated/dependent} under the generative model. However, as the variational approximation remains restricted to being fully factored, it may fail to fully capture the uncertainty about the latent variables. 

\paragraph{Gradient calculations.}
 We end this subsection by stating the expressions needed to calculate the gradients in the Gaussian example presented above. Throughout, we use the \emph{denominator-layout} notation for vector and matrix calculus and sometimes write $\smash{\particleReparametrised = \particleReparametrised_{1:D} = h_{\phi,x}^{-1}(z)}$ to simplify the notation. Thus,
\begin{align}
 \nabla_\theta \log \uTarget_{\theta,x}(z) 
 & = \smash{\varSigma^{-1}(z - \mu) \in \reals^{D},}\\
 \nabla_z \log \uTarget_{\theta,x}(z)
 & = \smash{\varSigma^{-1}(\mu - z) + x - z\in \reals^{D},} \label{eq:path_derivative_gamma}\\
 \nabla_z \log q_{\phi,x}(z) 
 & = \smash{-C^{-1}(z - Ax - b)}\\
 & = \smash{-C^{-1/2} \particleReparametrised\in \reals^{D}.} \label{eq:path_derivative_q}
\end{align}
Let $\smash{a_d \coloneqq [a_{d,1}, a_{d,2}, \dotsc, a_{d,D}]^\T \in \reals^{D\times1}}$ denote the column vector formed by the elements in the $d$th row of $A$. Then, letting $\odot$ denote elementwise multiplication and using the convention that addition or subtraction of the scalar $1$ is to be done elementwise,
\begin{align}
 \nabla_{a_d} \log q_{\phi,x}(z) 
 & = \smash{C^{-1}(z_d-  a_d^\T x - b_d)x}\\
 & = \smash{C^{-1/2} \particleReparametrised_d x \in \reals^{D}, \quad d \in \{1,\dotsc,D\},}\\
 \nabla_b \log q_{\phi,x}(z) 
 & = \smash{C^{-1}(z - Ax - b)}\\
 & = \smash{C^{-1/2} \particleReparametrised \in \reals^{D},}\\
 \nabla_c \log q_{\phi,x}(z) 
 & = \smash{C^{-1/2} (z - Ax - b) \odot C^{-1/2} (z - Ax - b) - 1}\\
 & = \smash{\particleReparametrised \odot\particleReparametrised - 1 \in \reals^{D},}
\end{align}
Furthermore, write $\smash{h_{\phi,x} = [h_{\phi,x,1}, \dotsc, h_{\phi,x,D}]^\T}$, i.e.\
\begin{align}
 h_{\phi,x,d}(\particleReparametrised) = z_d = a_d^\T x + b_d + \exp(c_d) \particleReparametrised_d,
\end{align}
and let ${\iota^{(d)} = [0,\dotsc,0,1,0,\dotsc,0]^\T \in \reals^{D}}$ be the vector whose entries are all $0$ except for the $d$th entry which is $1$. Then, for $\smash{d \in \{1,\dotsc,D\}}$, 
\begin{align}
 [\nabla_{a_{d'}} h_{\phi,x,d}](\particleReparametrised) 
 & = \smash{\ind\{d = d'\} x \in \reals^{D}, \quad d' \in \{1,\dotsc,D\},} \label{eq:a_d_path_derivative_gradient}\\
 [\nabla_b h_{\phi,x,d}](\particleReparametrised) 
 & = \smash{\iota^{(d)} \in \reals^{D},}\label{eq:b_path_derivative_gradient}\\
 [\nabla_c h_{\phi,x,d}](\particleReparametrised) 
  & = \smash{\exp(c_d)\particleReparametrised_d \iota^{(d)} \in \reals^{D}.} \label{eq:c_path_derivative_gradient}
\end{align}
Again writing $\smash{\particleReparametrised = h_{\phi,x}^{-1}(z)}$ implies that
\begin{align}
 \nabla_{\phi} [{\log} \circ {w_{\psi^{\mathrlap{\prime}},x}} \circ {h_{\phi,x}}]|_{\psi^\prime = \psi} (\particleReparametrised)
 & = [\nabla_\phi h_{\phi,x,1}, \dotsc, \nabla_\phi h_{\phi,x,D}](\particleReparametrised) \nabla_{z} \log w_{\psi,x}(z),
\end{align}
so that, letting $\smash{[\nabla_{z} \log w_{\psi,x}(z)]_d}$ denote the $d$th element of the vector $\smash{\nabla_{z} \log w_{\psi,x}(z)}$,
\begin{align}
 \nabla_{a_d} [{\log} \circ {w_{\psi^{\mathrlap{\prime}},x}} \circ {h_{\phi,x}}]|_{\psi^\prime = \psi} (\particleReparametrised)
 & = [\nabla_{z} \log w_{\psi,x}(z)]_d x,\\
  \nabla_{b} [{\log} \circ {w_{\psi^{\mathrlap{\prime}},x}} \circ {h_{\phi,x}}]|_{\psi^\prime = \psi} (\particleReparametrised)
 & = \nabla_z \log w_{\psi,x}(z),\\
 \nabla_{c} [{\log} \circ {w_{\psi^{\mathrlap{\prime}},x}} \circ {h_{\phi,x}}]|_{\psi^\prime = \psi} (\particleReparametrised)
 & = \particleReparametrised \odot C^{1/2} \nabla_z \log w_{\psi,x}(z).
\end{align}
From this, since
\begin{align}
 \nabla_{\phi} [{\log} \circ {w_{\psi,x}} \circ {h_{\phi,x}}](\particleReparametrised)
 & = \nabla_{\phi} [{\log} \circ {w_{\psi^{\mathrlap{\prime}},x}} \circ {h_{\phi,x}}]|_{\psi^\prime = \psi}(\particleReparametrised) - \nabla_\phi \log q_{\phi,x}(z),
\end{align}
we have that
\begin{align}
 \nabla_{a_d} [{\log} \circ {w_{\psi,x}} \circ {h_{\phi,x}}] (\particleReparametrised)
 & = ([\nabla_{z} \log w_{\psi,x}(z)]_d - C^{-1/2} \particleReparametrised_d) x,\\
  \nabla_{b} [{\log} \circ {w_{\psi,x}} \circ {h_{\phi,x}}] (\particleReparametrised)
 & = \nabla_z \log w_{\psi,x}(z) - C^{-1/2}\particleReparametrised,\\
 \nabla_{c} [{\log} \circ {w_{\psi,x}} \circ {h_{\phi,x}}] (\particleReparametrised)
 & = \particleReparametrised \odot C^{1/2} \nabla_z \log w_{\psi,x}(z) - \particleReparametrised \odot \particleReparametrised + 1.
\end{align}

\paragraph{Impact of the reparametrisation.} We end this subsection by briefly illustrating the impact of the reparametrisation trick combined with the identity from \citet{tucker2019reparametrised} which was given in Lemma~\ref{lem:tucker}. Recall that this approach yields $\phi$-gradients that are expressible as integrals of path-derivative functions $\smash{\blacktriangledown_{\psi,x} \coloneqq  \nabla_{\phi} [{\log} \circ {w_{\psi^{\mathrlap{\prime}},x}} \circ {h_{\phi,x}}]|_{\psi^\prime = \psi} \circ {h_{\phi,x}^{-1}}}$, Thus, if there exists a value $\phi$ such that $\smash{q_{\phi,x} = \target_{\theta,x}}$ then $w_{\psi,x} \propto \target_{\theta,x}/q_{\phi,x} \equiv 1$ is constant so that we obtain zero-variance $\phi$-gradients \citep[see, e.g.,][for a discussion on this]{roeder2017sticking}.

For simplicity, assume that $\varSigma = \iMat$ and recall that we then have $\smash{q_{\phi^{\mathrlap{\star}}\,,x} = \target_{\theta,x}}$ if the values $(A, b, C)$ implied by $\phi^\star$ are $(A^{\mathrlap{\star}}\,, b^{\mathrlap{\star}}\,, C^\star)=(\frac{1}{2}\iMat, \frac{1}{2} \mu, \frac{1}{2}\iMat)$. 

By \eqref{eq:path_derivative_gamma} and \eqref{eq:path_derivative_q}, and with the usual convention $\smash{\particleReparametrised = h_{\phi,x}^{-1}(z)}$, we then have
\begin{align}
 \nabla_{z} \log w_{\psi,x}(z) 
 & = (x+\mu) - 2z + C^{-1} (z - Ax - b)\\
 & = 2 [(A^\star x+ b^\star) - (Ax + b) + C^{-1/2}(C^\star - C) \particleReparametrised]. \label{eq:log_weight_path_derivative_illustration}
\end{align}
Note that the only source of randomness in this expression is the multivariate normal random variable $\particleReparametrised$. Thus, by \eqref{eq:a_d_path_derivative_gradient} and \eqref{eq:b_path_derivative_gradient}, for \emph{any} values of $A$ and $b$ and \emph{any} $K\geq1$, the variance of the $A$- and $b$-gradient portion of \gls{AISLEKL}\slash{}\gls{IWAESTL} and \gls{AISLECHISQ}\slash{}\gls{IWAEDREG} goes to zero as $C \to C^\star = \frac{1}{2}\iMat$. In other words, in this model, these `score-function free' $\phi$-gradients achieve (near) zero variance for the parameters governing the proposal mean as soon as the variance-parameters fall within a neighbourhood of their optimal values. Furthermore, \eqref{eq:c_path_derivative_gradient} combined with \eqref{eq:log_weight_path_derivative_illustration} shows that for \emph{any} $K \geq 1$, the variance of the $C$-gradient portion also goes to zero as $(A, b, C) \to (A^{\mathrlap{\star}}\,, b^{\mathrlap{\star}}\,, C^\star)$. A more thorough analysis of the benefits of reparametrisation-trick gradients in Gaussian settings is carried out in \citet{xu2019variance}.

\subsection{Simulations}

\paragraph{Setup.}
We end this section by empirically comparing the algorithms from Subsection~\ref{app:subsec:algorithms}. We run each of these algorithms for a varying number of particles, $K \in \{1,10,100\}$, and varying model dimensions, $D \in \{2,5,10\}$. Each of these configurations is repeated independently $250$ times. Each time using a new synthetic data set consisting of $N = 25$ observations sampled from the generative model after generating a new `true' prior mean vector as $\mu \sim \dN(0, \iMat)$. Since all the algorithms share the same $\theta$-gradient, we focus only on the optimisation of $\phi$ and thus simply fix $\theta \coloneqq \theta^\ML$ throughout. We show results for the following model settings.
\begin{itemize}

 \item \textbf{Figure~\ref{fig:diagonal}.}  The generative model is specified via $\varSigma = \iMat$. In this case, there exists a value $\phi^\star$ of $\phi$ such that $q_{\phi,x}(z) = \target_{\theta,x}(z)$. Note that this corresponds to Scenario~\ref{enum:sufficiently_expressive_variational_family} in Subsection~\ref{app:subsec:inclusive_vs_exclusive}. 
 
 
 \item \textbf{Figure~\ref{fig:nondiagonal}.} The generative model is specified via $\smash{\varSigma = (0.95^{\lvert d-d'\rvert +1})_{(d,d') \in \{1,\dotsc,D\}^2}}$. Note that in this case, the fully-factored variational approximation cannot fully mimic the dependence structure of the latent variables under the generative model. That is, in this case, $q_{\phi,x}(z) \neq \target_{\theta,x}(z)$ for any values of $\phi$. Note that this corresponds to Scenario~\ref{enum:insufficiently_expressive_variational_family} in Subsection~\ref{app:subsec:inclusive_vs_exclusive}.
 
 
\end{itemize} 

To initialise the gradient-ascent algorithm, we draw each component of the initial values  $\phi_0$ of $\phi$ \gls{IID} according to a standard normal distribution. We use both plain stochastic gradient-ascent with the gradients normalised to have unit $\mathbb{L}_1$-norm (Figures~\ref{fig:diagonal:sga_standard_weights}, \ref{fig:diagonal:sga_regularised_weights}, \ref{fig:nondiagonal:sga_standard_weights}, \ref{fig:nondiagonal:sga_regularised_weights}) and \gls{ADAM}  \citep{kingma2015adam} with default parameter values (Figures~\ref{fig:diagonal:adam_standard_weights}, \ref{fig:diagonal:adam_regularised_weights}, \ref{fig:nondiagonal:adam_standard_weights}, \ref{fig:nondiagonal:adam_regularised_weights}). In each case, we also show results for the `regularised importance weights' strategy from Subsection~\ref{subsec:bias-variance_trade-offs} with tuning parameter $\eta = 0.8$ (Figures~\ref{fig:diagonal:sga_regularised_weights}, \ref{fig:diagonal:adam_regularised_weights}, \ref{fig:nondiagonal:sga_regularised_weights}, \ref{fig:nondiagonal:adam_regularised_weights}). The total number of iterations is $\smash{10,000}$; in each case, the learning-rate parameters at the $i$th step are $\smash{i^{-1/2}}$. 

We also ran the algorithms in each of the above-mentioned scenarios with fixed values of $c_d$, e.g.\ as in \citet{rainforth2018tighter, tucker2019reparametrised}. However, we omit the results as this did not significantly change the relative performance of the different algorithms. For the same reason, we omit results related to the optimisation of $A$ and $C$.

\newgeometry{left=1cm, right=1cm, top=1cm,bottom=1cm}
\thispagestyle{empty}
\begin{figure}[!ht]
\begin{subfigure}{0.5\linewidth}

\centering
\includegraphics[scale=0.57, trim=0.5cm 1cm 0cm 0cm]{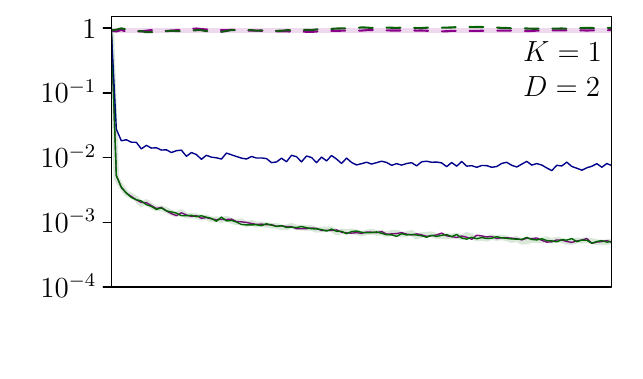}
\includegraphics[scale=0.57, trim=1.5cm 1cm 0cm 0cm]{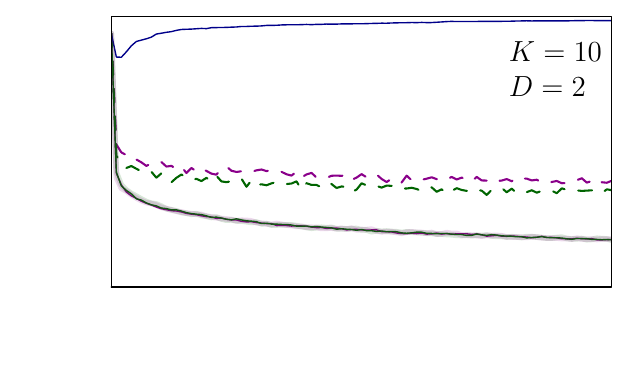}
\includegraphics[scale=0.57, trim=1.5cm 1cm 0cm 0cm]{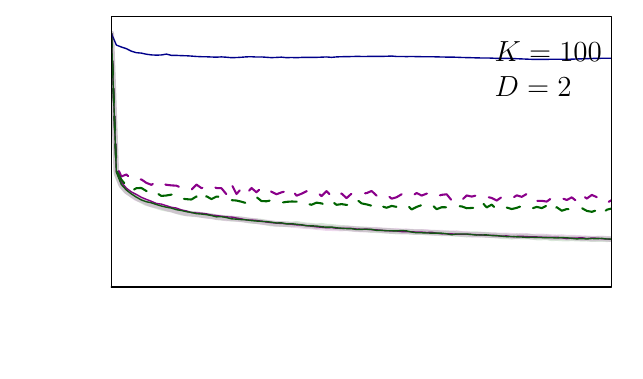}\\
\includegraphics[scale=0.57, trim=0.5cm 1cm 0cm 0cm]{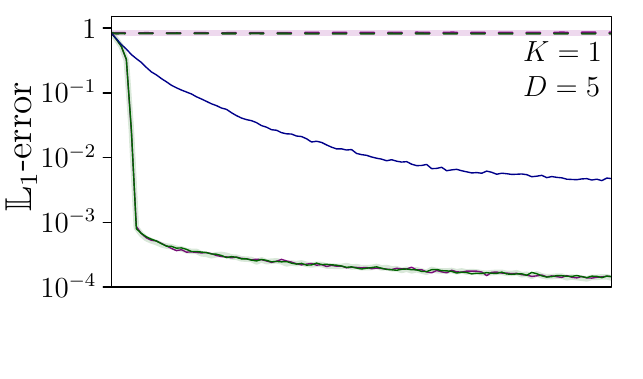}
\includegraphics[scale=0.57, trim=1.5cm 1cm 0cm 0cm]{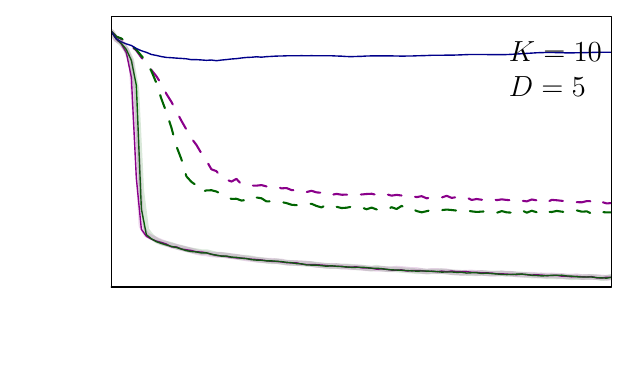}
\includegraphics[scale=0.57, trim=1.5cm 1cm 0cm 0cm]{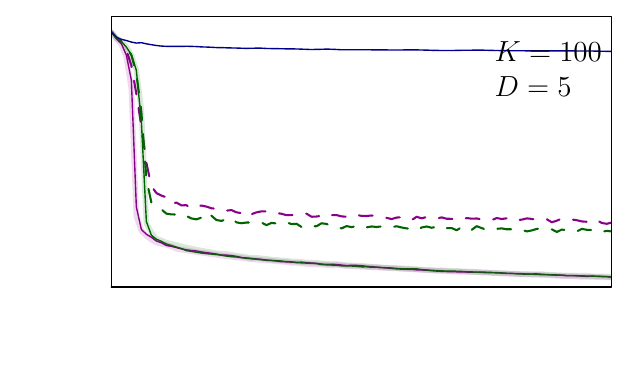}\\
\includegraphics[scale=0.57, trim=0.5cm 0cm 0cm 0cm]{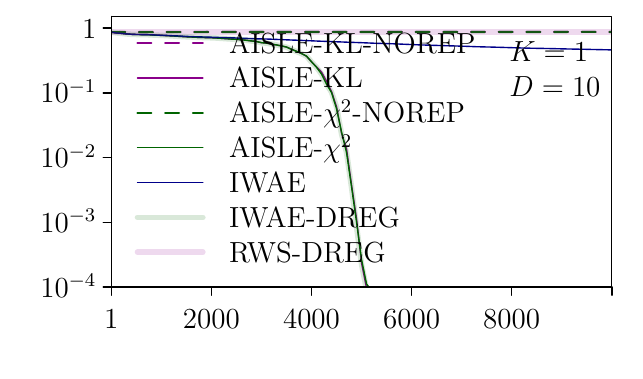}
\includegraphics[scale=0.57, trim=1.5cm 0cm 0cm 0cm]{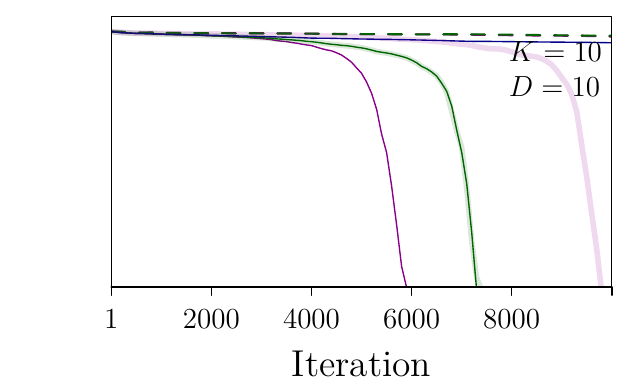}
\includegraphics[scale=0.57, trim=1.5cm 0cm 0cm 0cm]{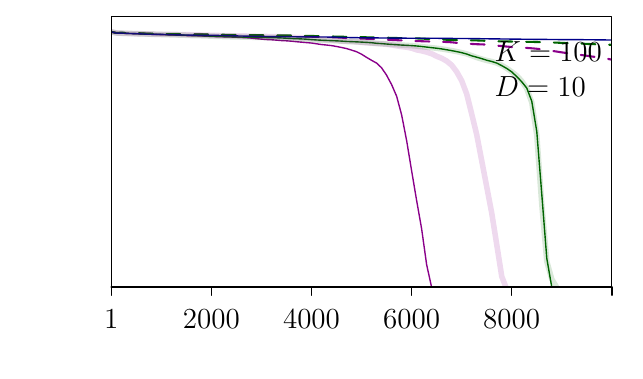} 

\caption{Gradient ascent with standard weights.}
 \label{fig:diagonal:sga_standard_weights}

\end{subfigure}
\begin{subfigure}{0.5\linewidth}

 \centering
\includegraphics[scale=0.57, trim=0.5cm 1cm 0cm 0cm]{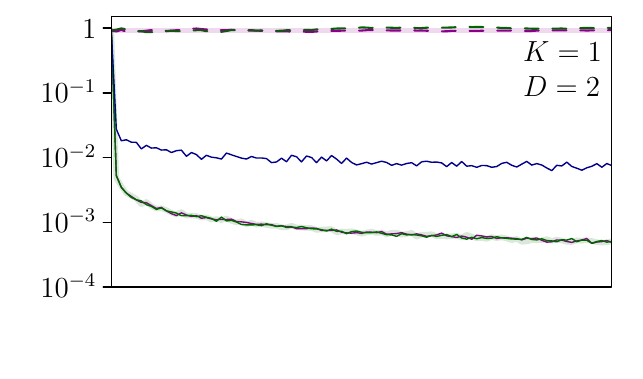}
\includegraphics[scale=0.57, trim=1.5cm 1cm 0cm 0cm]{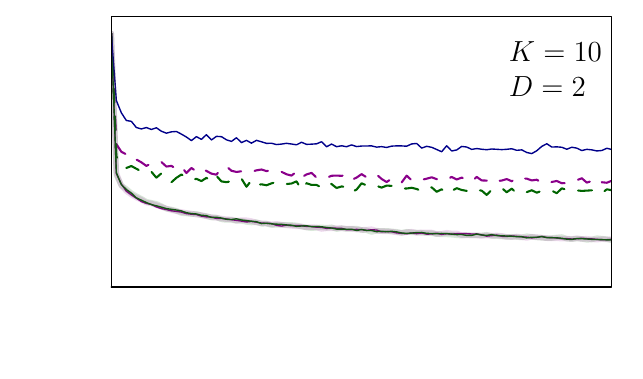}
\includegraphics[scale=0.57, trim=1.5cm 1cm 0.5cm 0cm]{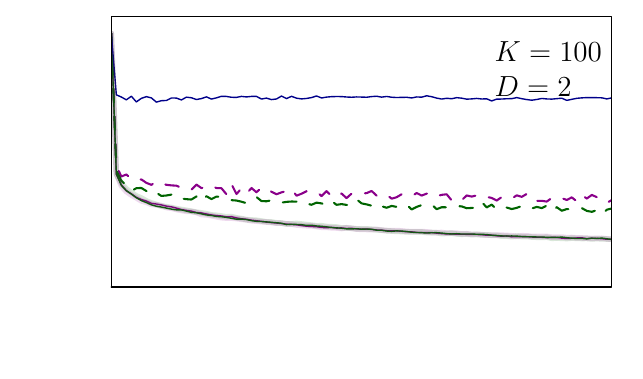}\\
\includegraphics[scale=0.57, trim=0.5cm 1cm 0cm 0cm]{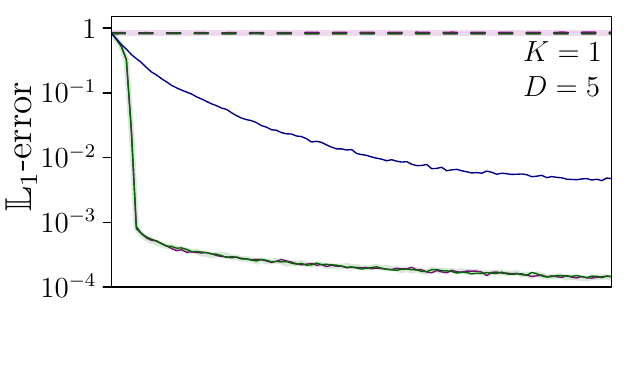}
\includegraphics[scale=0.57, trim=1.5cm 1cm 0cm 0cm]{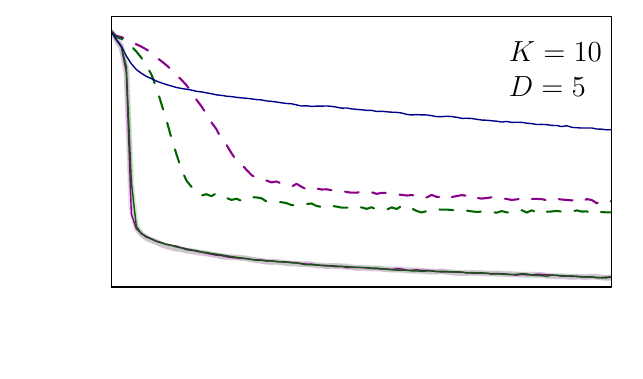}
\includegraphics[scale=0.57, trim=1.5cm 1cm 0.5cm 0cm]{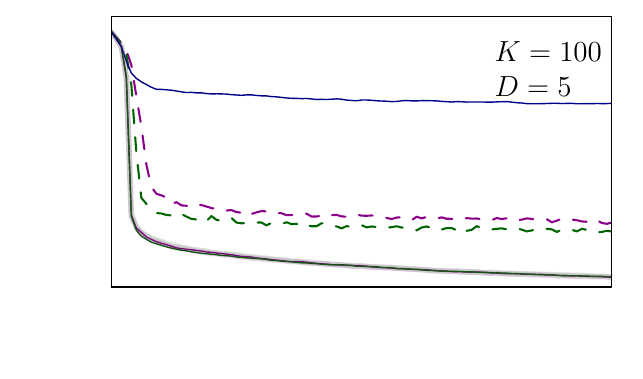}\\
\includegraphics[scale=0.57, trim=0.5cm 0cm 0cm 0cm]{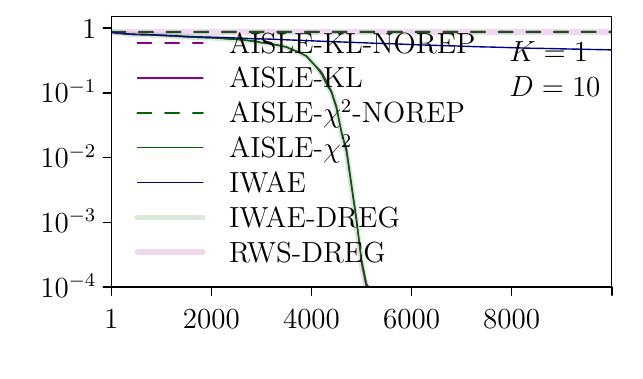}
\includegraphics[scale=0.57, trim=1.5cm 0cm 0cm 0cm]{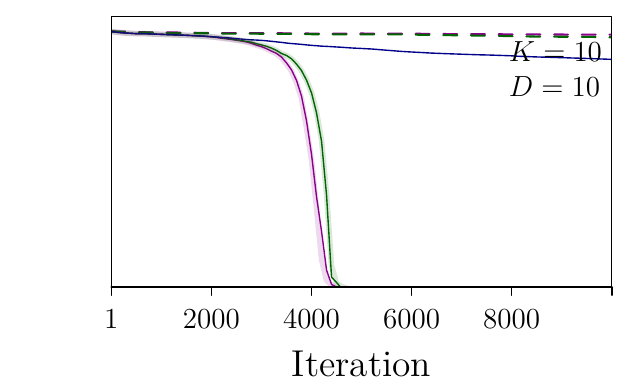}
\includegraphics[scale=0.57, trim=1.5cm 0cm 0.5cm 0cm]{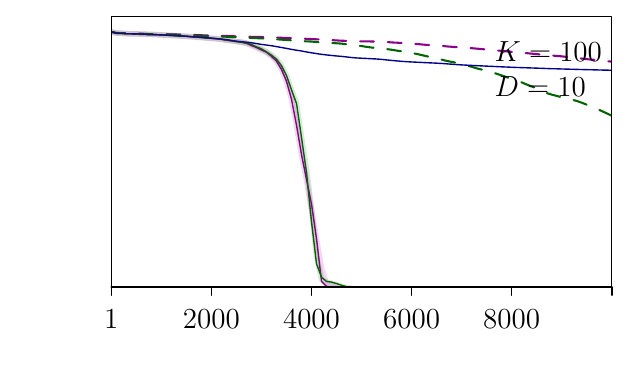}

\caption{Gradient ascent with regularised weights.}
\label{fig:diagonal:sga_regularised_weights}

\end{subfigure}

\begin{subfigure}{0.5\linewidth}
 

\centering
\includegraphics[scale=0.57, trim=0.5cm 1cm 0cm 0cm]{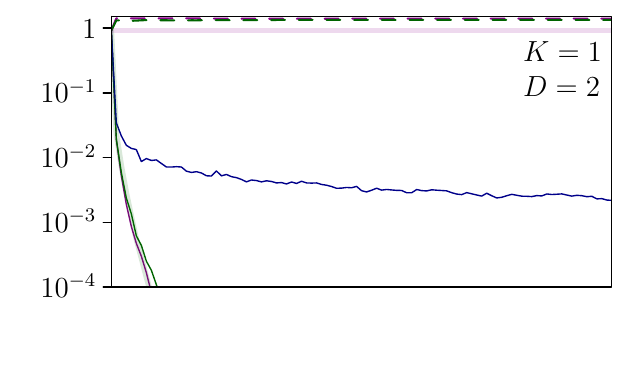}
\includegraphics[scale=0.57, trim=1.5cm 1cm 0cm 0cm]{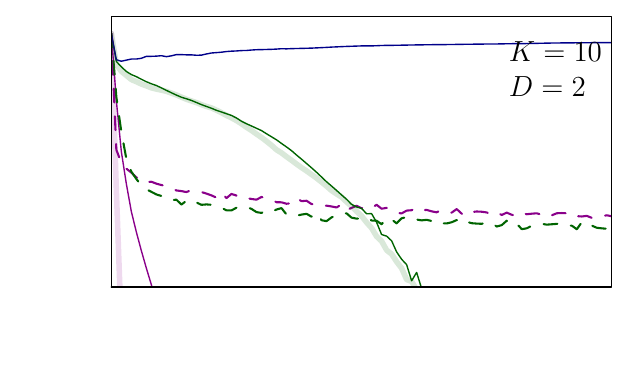}
\includegraphics[scale=0.57, trim=1.5cm 1cm 0cm 0cm]{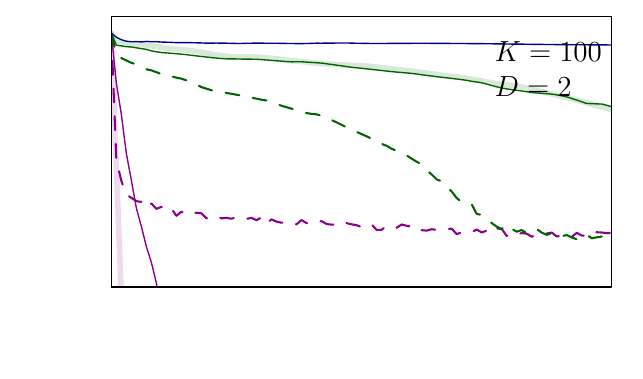}\\
\includegraphics[scale=0.57, trim=0.5cm 1cm 0cm 0cm]{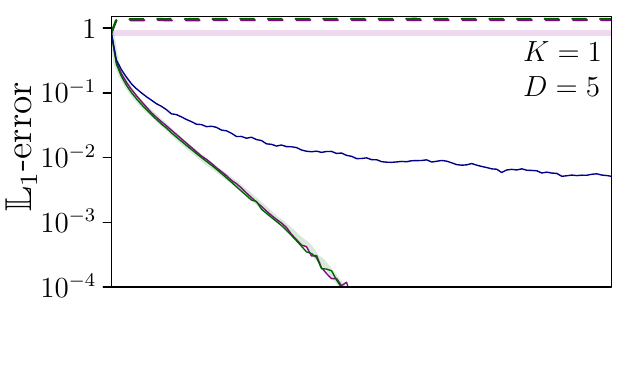}
\includegraphics[scale=0.57, trim=1.5cm 1cm 0cm 0cm]{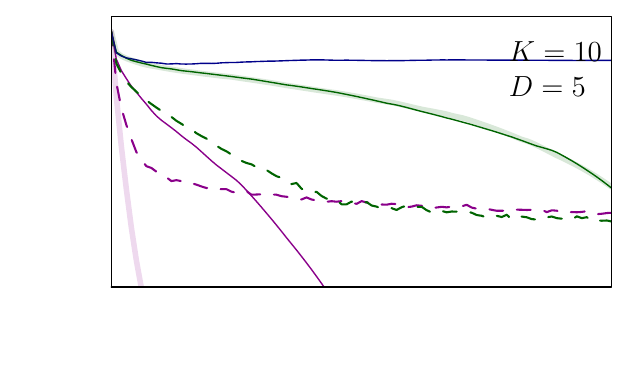}
\includegraphics[scale=0.57, trim=1.5cm 1cm 0cm 0cm]{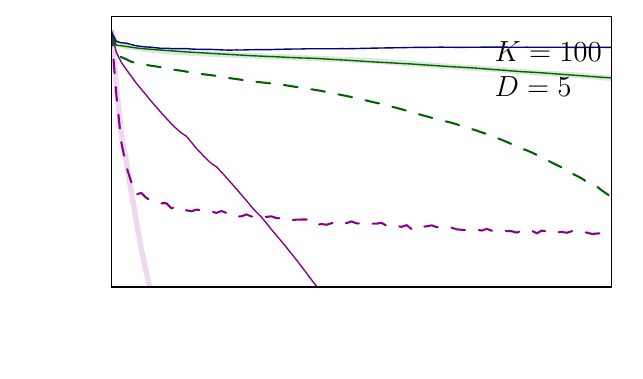}\\
\includegraphics[scale=0.57, trim=0.5cm 0cm 0cm 0cm]{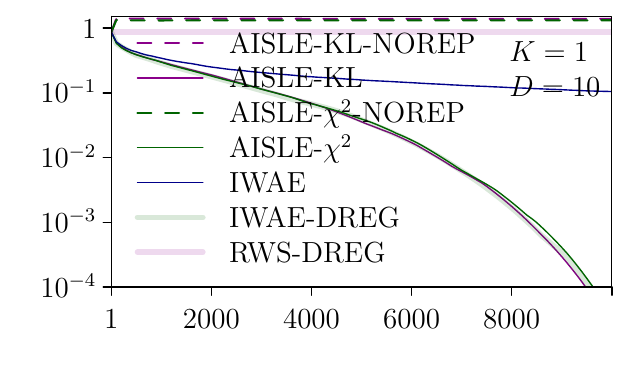}
\includegraphics[scale=0.57, trim=1.5cm 0cm 0cm 0cm]{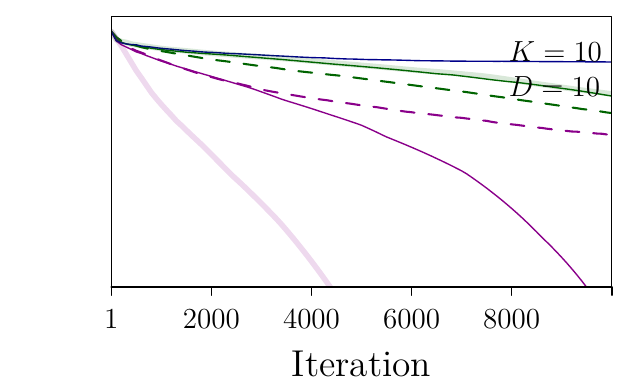}
\includegraphics[scale=0.57, trim=1.5cm 0cm 0cm 0cm]{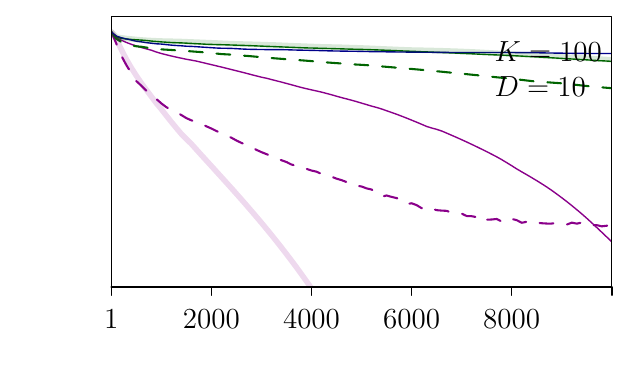} 

\caption{\gls{ADAM} with standard weights.}
\label{fig:diagonal:adam_standard_weights}
\end{subfigure}
\begin{subfigure}{0.5\linewidth}


\centering
\includegraphics[scale=0.57, trim=0.5cm 1cm 0cm 0cm]{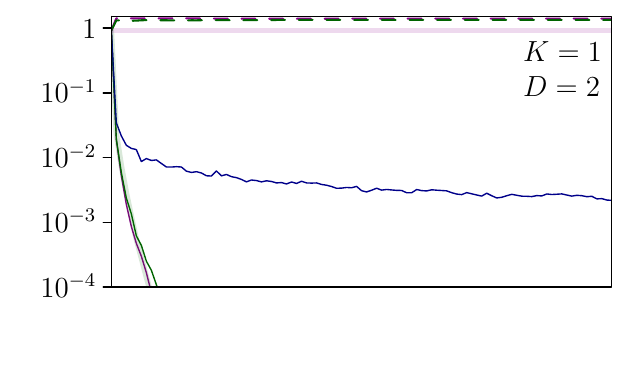}
\includegraphics[scale=0.57, trim=1.5cm 1cm 0cm 0cm]{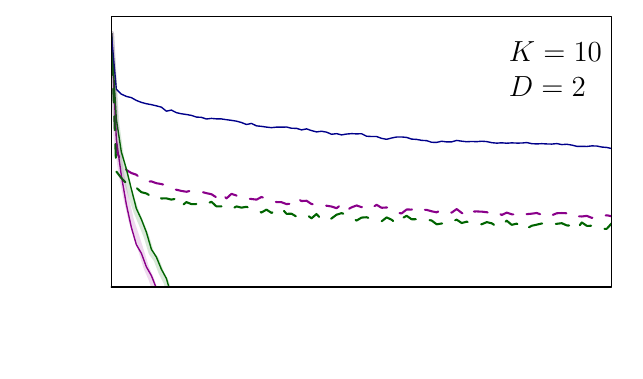}
\includegraphics[scale=0.57, trim=1.5cm 1cm 0.5cm 0cm]{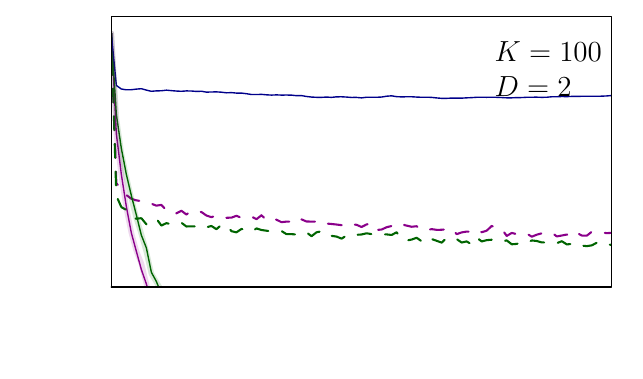}\\
\includegraphics[scale=0.57, trim=0.5cm 1cm 0cm 0cm]{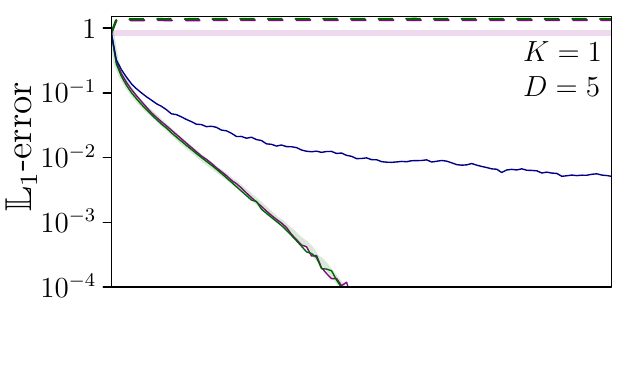}
\includegraphics[scale=0.57, trim=1.5cm 1cm 0cm 0cm]{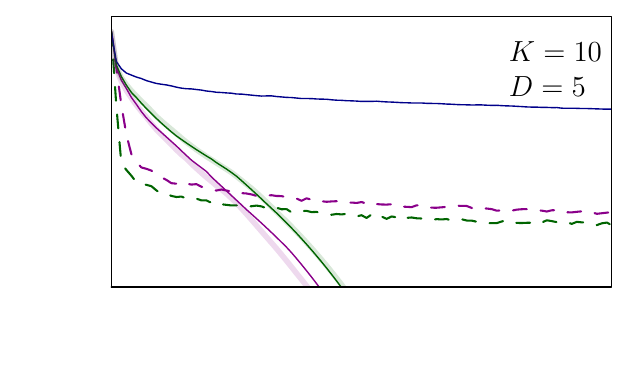} 
\includegraphics[scale=0.57, trim=1.5cm 1cm 0.5cm 0cm]{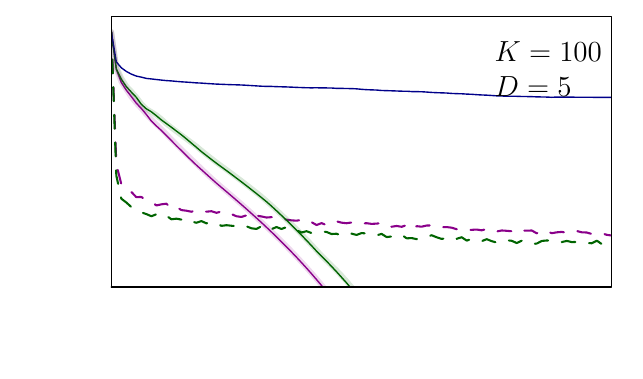}\\
\includegraphics[scale=0.57, trim=0.5cm 0cm 0cm 0cm]{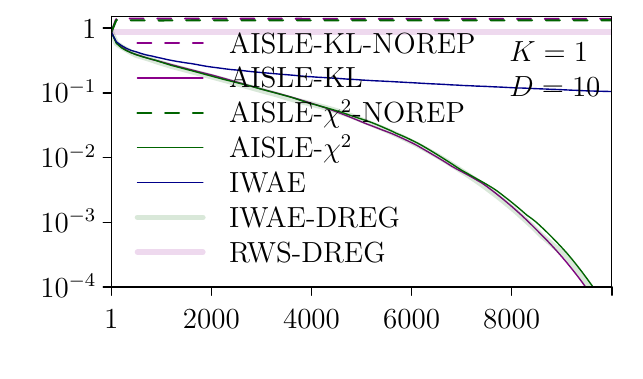}
\includegraphics[scale=0.57, trim=1.5cm 0cm 0cm 0cm]{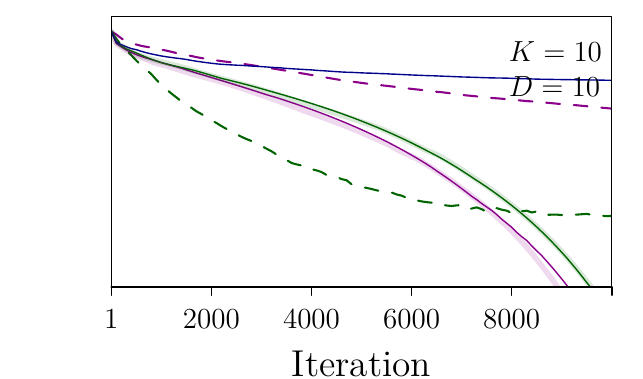}
\includegraphics[scale=0.57, trim=1.5cm 0cm 0.5cm 0cm]{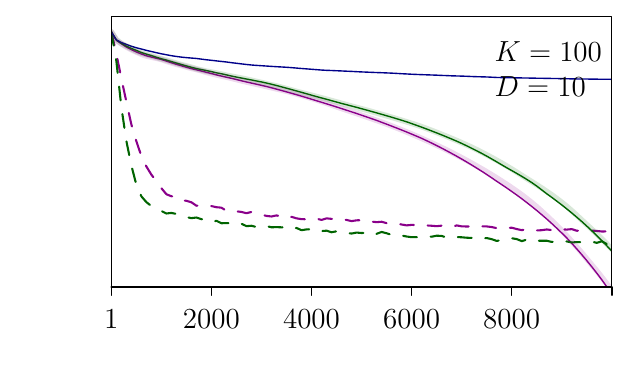} 

\caption{\gls{ADAM} with regularised weights.}
\label{fig:diagonal:adam_regularised_weights}

\end{subfigure}

    \caption{Average $\mathbb{L}_1$-error of the estimates of the parameters $b = b_{1:D}$ governing the mean of the Gaussian variational family. The average is taken over the $D$ components of $b$ and the figure displays the median error at each iteration over $100$ independent runs of each algorithm, each using a different data set consisting of $25$ observations sampled from the model. Note the logarithmic scaling on the second axis. Here, the covariance matrix \emph{$\varSigma = \iMat$ is diagonal.} }
  \label{fig:diagonal}

  \vspace*{\floatsep}
  
\begin{subfigure}{0.5\linewidth}

\centering
\includegraphics[scale=0.57, trim=0.5cm 1cm 0cm 0cm]{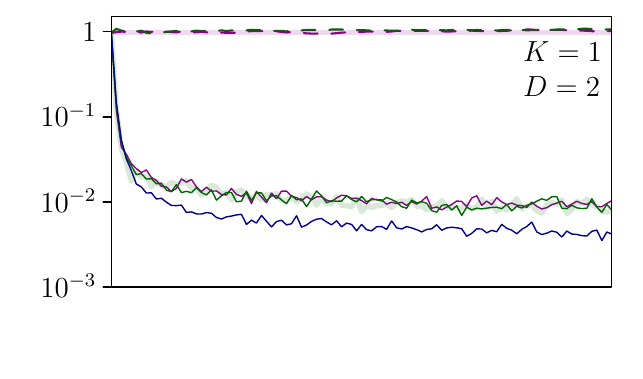}
\includegraphics[scale=0.57, trim=1.5cm 1cm 0cm 0cm]{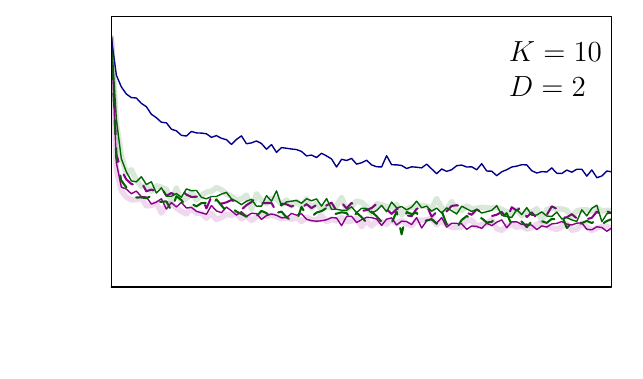}
\includegraphics[scale=0.57, trim=1.5cm 1cm 0cm 0cm]{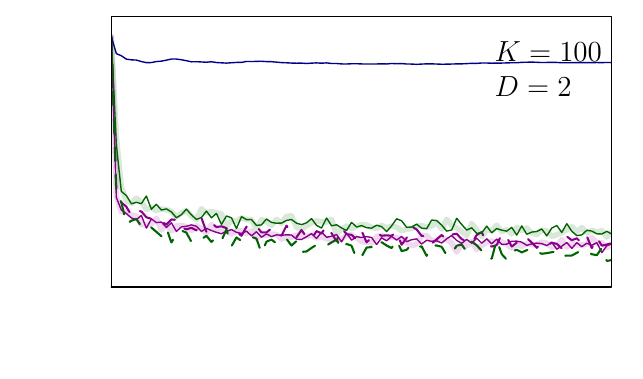}\\
\includegraphics[scale=0.57, trim=0.5cm 1cm 0cm 0cm]{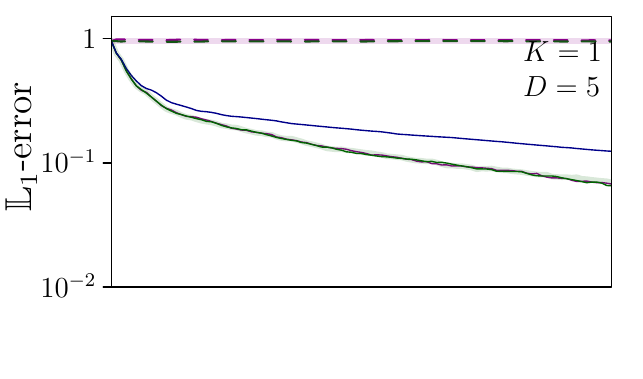}
\includegraphics[scale=0.57, trim=1.5cm 1cm 0cm 0cm]{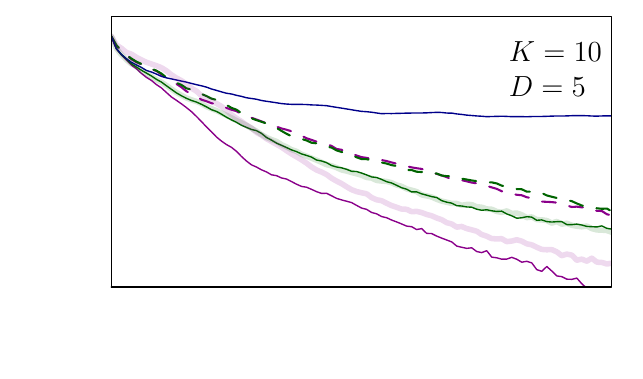}
\includegraphics[scale=0.57, trim=1.5cm 1cm 0cm 0cm]{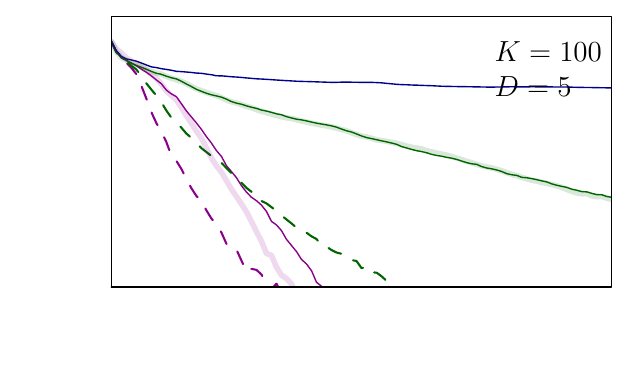}\\
\includegraphics[scale=0.57, trim=0.5cm 0cm 0cm 0cm]{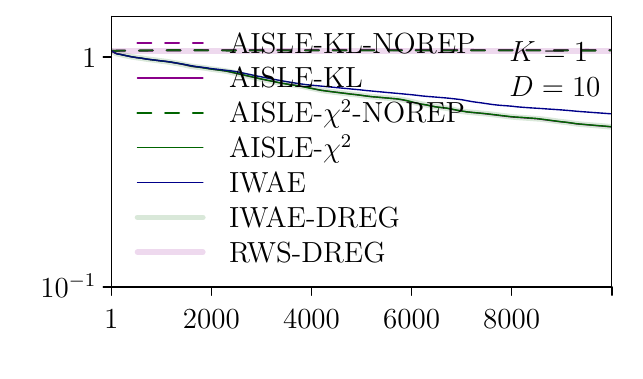}
\includegraphics[scale=0.57, trim=1.5cm 0cm 0cm 0cm]{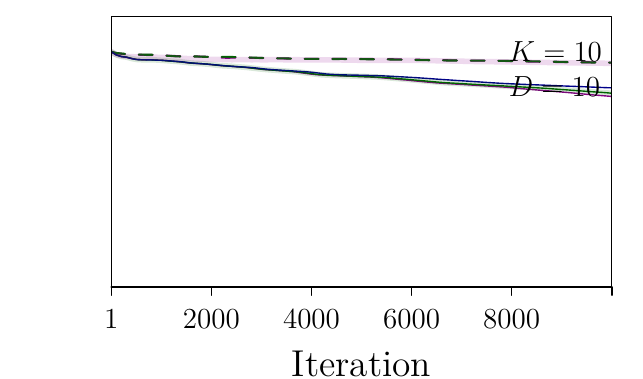}
\includegraphics[scale=0.57, trim=1.5cm 0cm 0cm 0cm]{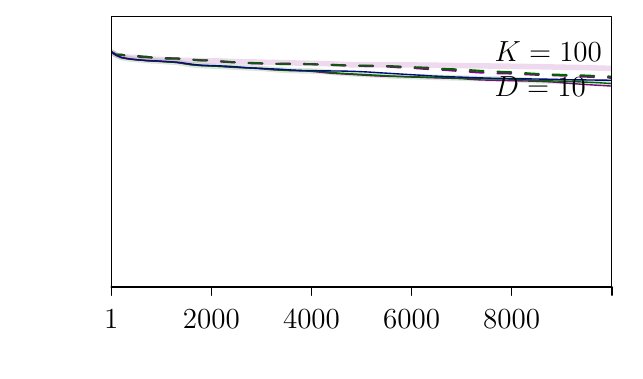} 

\caption{Gradient ascent with standard weights.}
 \label{fig:nondiagonal:sga_standard_weights}

\end{subfigure}
\begin{subfigure}{0.5\linewidth}

 \centering
\includegraphics[scale=0.57, trim=0.5cm 1cm 0cm 0cm]{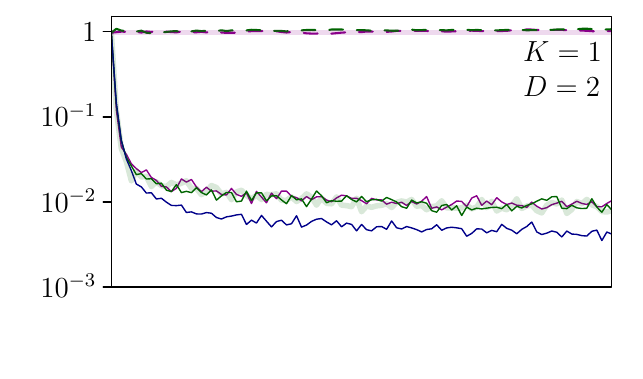}
\includegraphics[scale=0.57, trim=1.5cm 1cm 0cm 0cm]{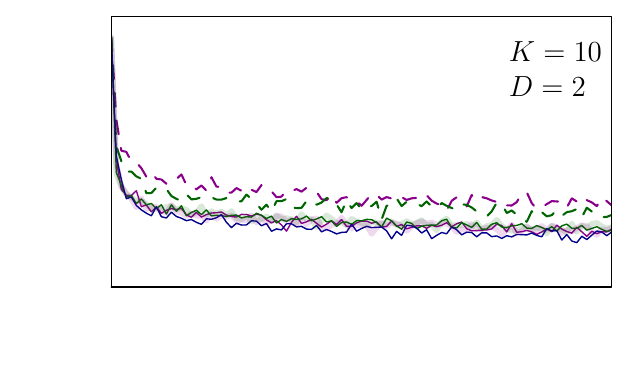}
\includegraphics[scale=0.57, trim=1.5cm 1cm 0.5cm 0cm]{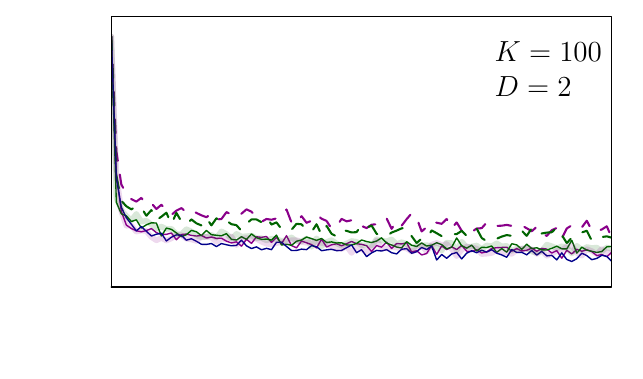}\\
\includegraphics[scale=0.57, trim=0.5cm 1cm 0cm 0cm]{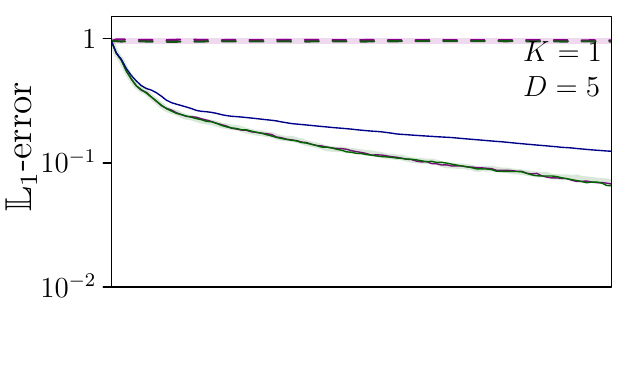}
\includegraphics[scale=0.57, trim=1.5cm 1cm 0cm 0cm]{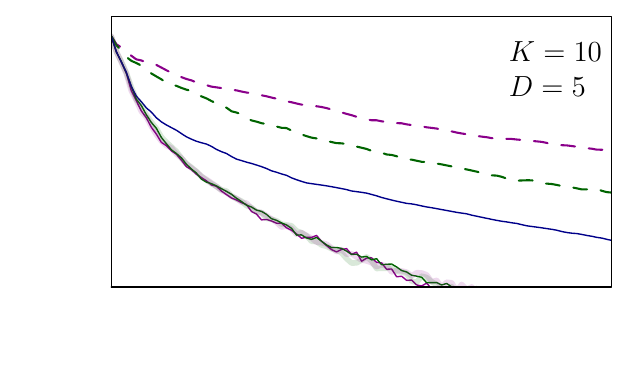}
\includegraphics[scale=0.57, trim=1.5cm 1cm 0.5cm 0cm]{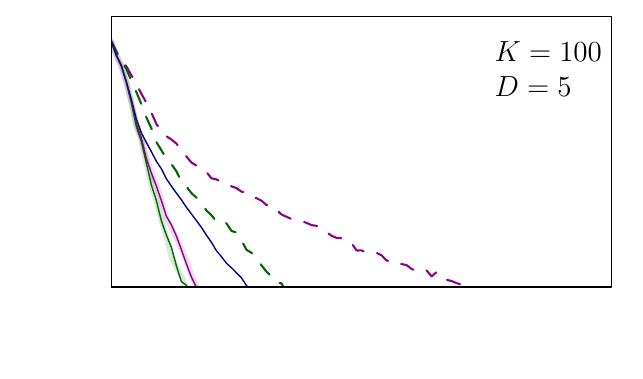}\\
\includegraphics[scale=0.57, trim=0.5cm 0cm 0cm 0cm]{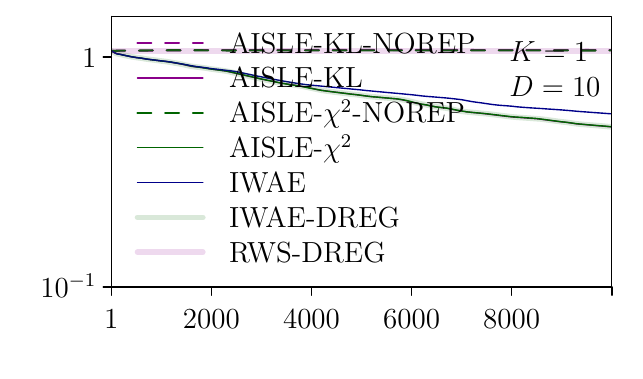}
\includegraphics[scale=0.57, trim=1.5cm 0cm 0cm 0cm]{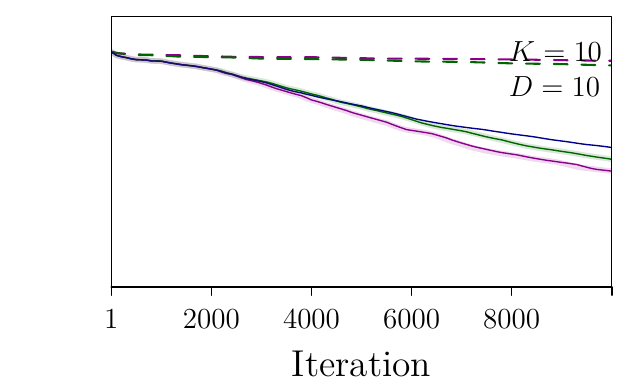}
\includegraphics[scale=0.57, trim=1.5cm 0cm 0.5cm 0cm]{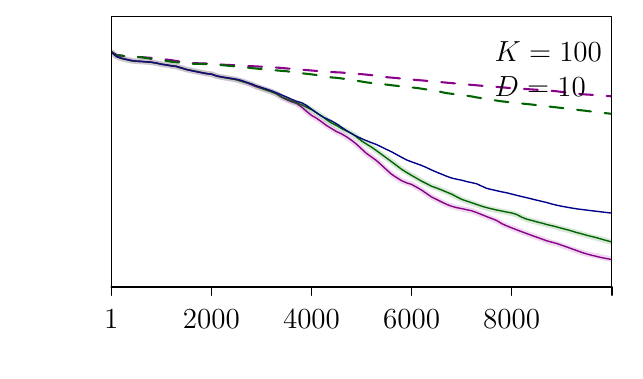} 

\caption{Gradient ascent with regularised weights.}
 \label{fig:nondiagonal:sga_regularised_weights}

\end{subfigure}

\begin{subfigure}{0.5\linewidth}
 

\centering
\includegraphics[scale=0.57, trim=0.5cm 1cm 0cm 0cm]{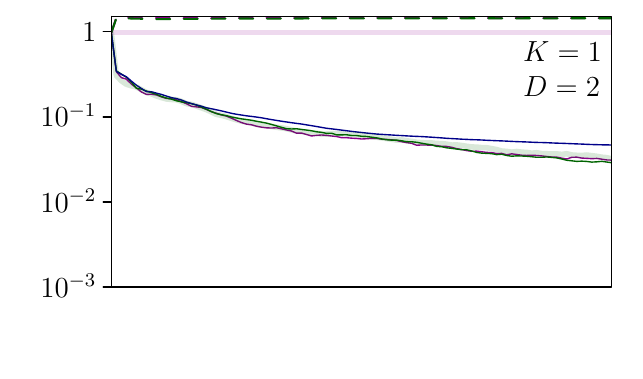}
\includegraphics[scale=0.57, trim=1.5cm 1cm 0cm 0cm]{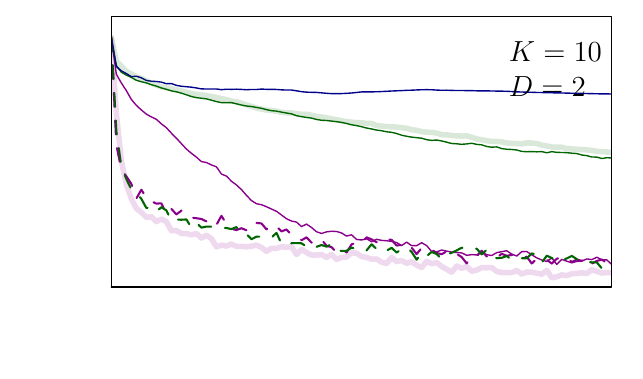}
\includegraphics[scale=0.57, trim=1.5cm 1cm 0cm 0cm]{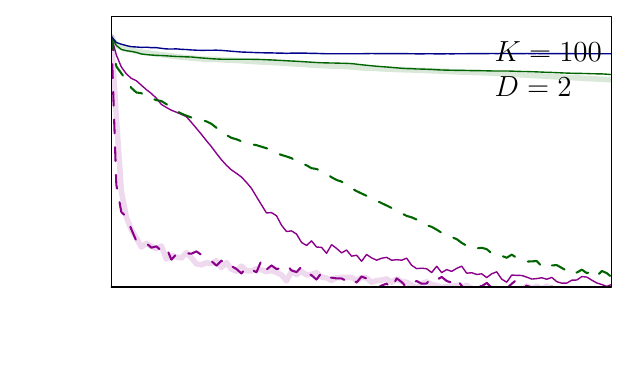}\\
\includegraphics[scale=0.57, trim=0.5cm 1cm 0cm 0cm]{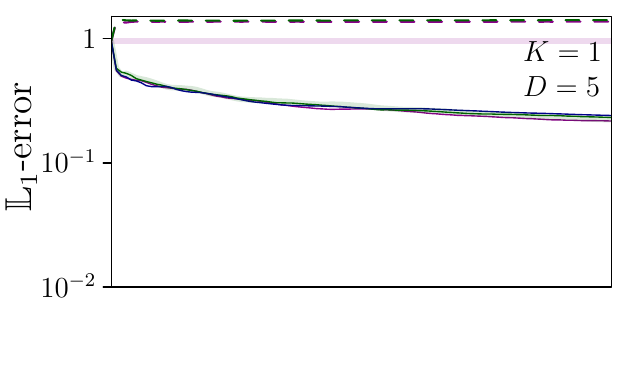}
\includegraphics[scale=0.57, trim=1.5cm 1cm 0cm 0cm]{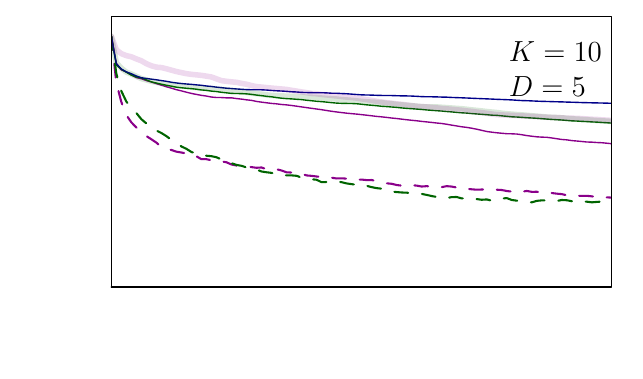}
\includegraphics[scale=0.57, trim=1.5cm 1cm 0cm 0cm]{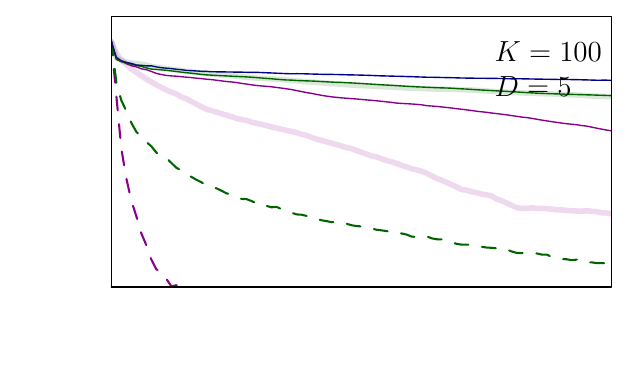}\\
\includegraphics[scale=0.57, trim=0.5cm 0cm 0cm 0cm]{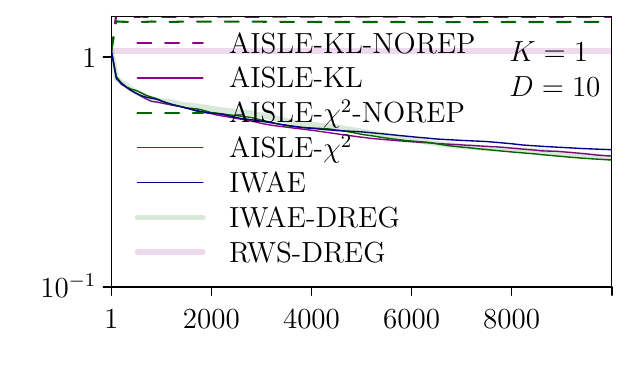}
\includegraphics[scale=0.57, trim=1.5cm 0cm 0cm 0cm]{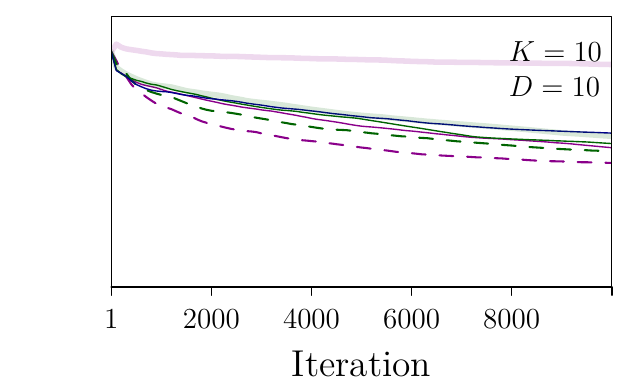}
\includegraphics[scale=0.57, trim=1.5cm 0cm 0cm 0cm]{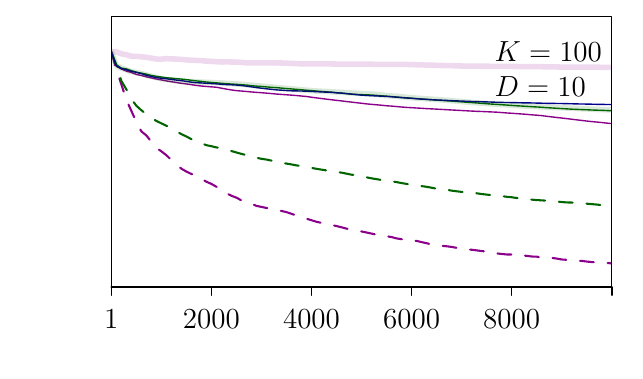}

\caption{\gls{ADAM} with standard weights.}
\label{fig:nondiagonal:adam_standard_weights}
\end{subfigure}
\begin{subfigure}{0.5\linewidth}


\centering
\includegraphics[scale=0.57, trim=0.5cm 1cm 0cm 0cm]{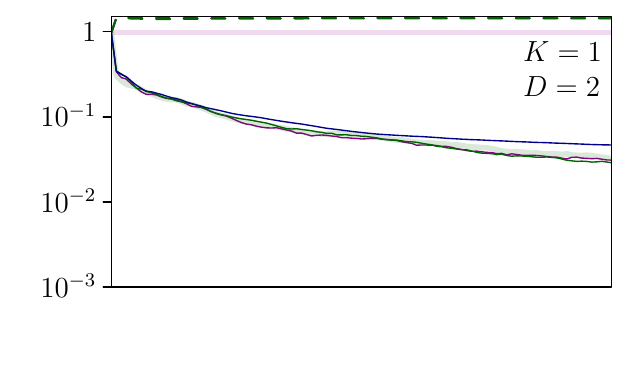}
\includegraphics[scale=0.57, trim=1.5cm 1cm 0cm 0cm]{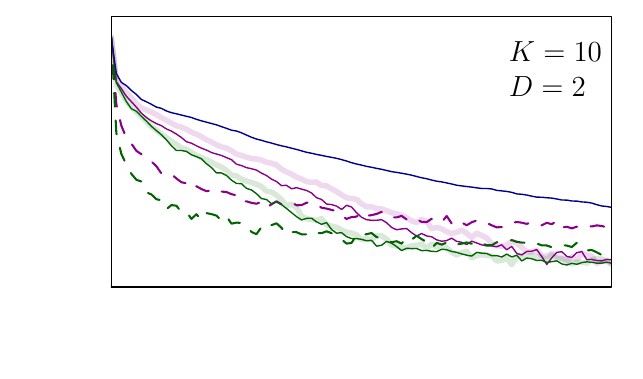}
\includegraphics[scale=0.57, trim=1.5cm 1cm 0.5cm 0cm]{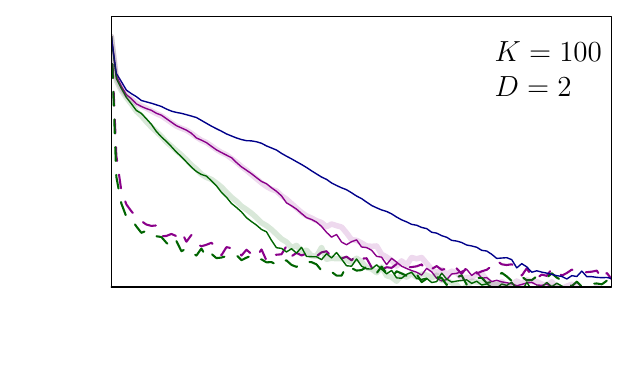}\\
\includegraphics[scale=0.57, trim=0.5cm 1cm 0cm 0cm]{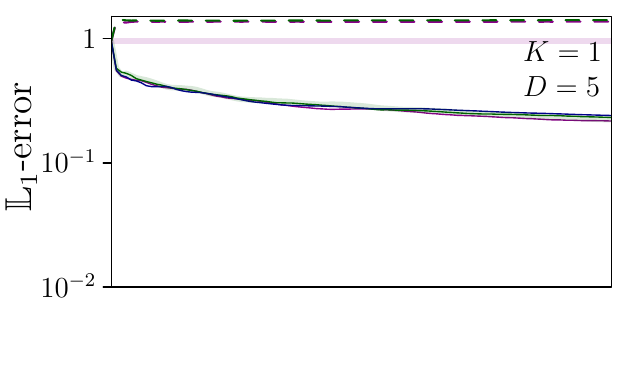}
\includegraphics[scale=0.57, trim=1.5cm 1cm 0cm 0cm]{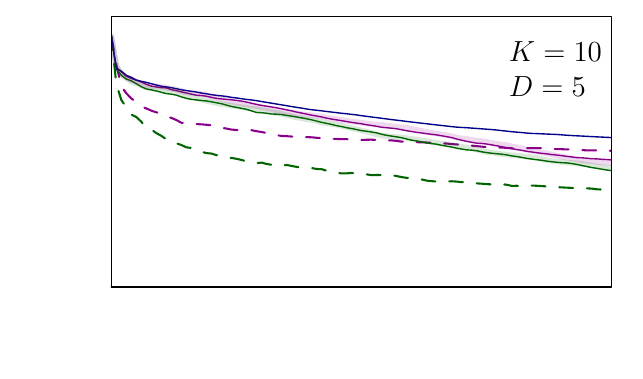} 
\includegraphics[scale=0.57, trim=1.5cm 1cm 0.5cm 0cm]{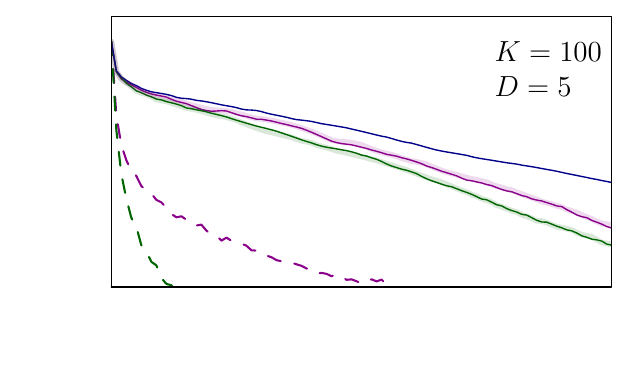}\\
\includegraphics[scale=0.57, trim=0.5cm 0cm 0cm 0cm]{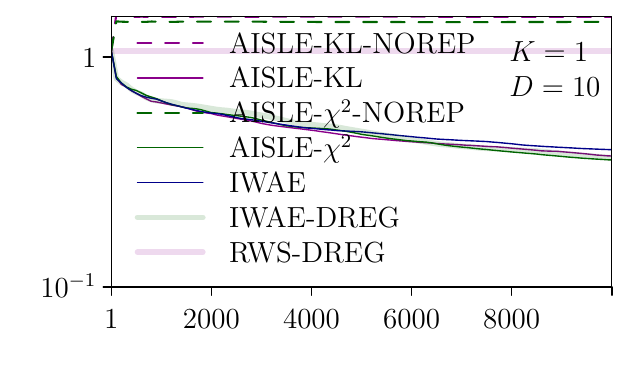}
\includegraphics[scale=0.57, trim=1.5cm 0cm 0cm 0cm]{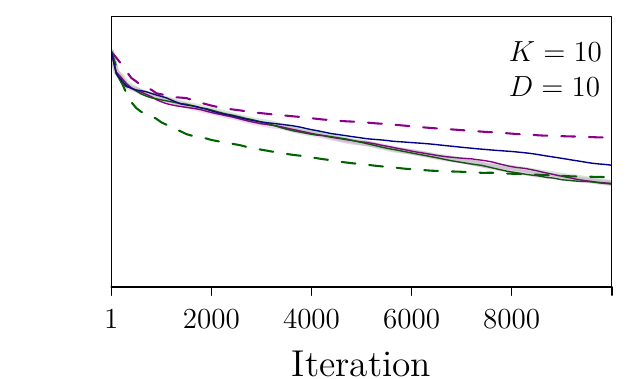}
\includegraphics[scale=0.57, trim=1.5cm 0cm 0.5cm 0cm]{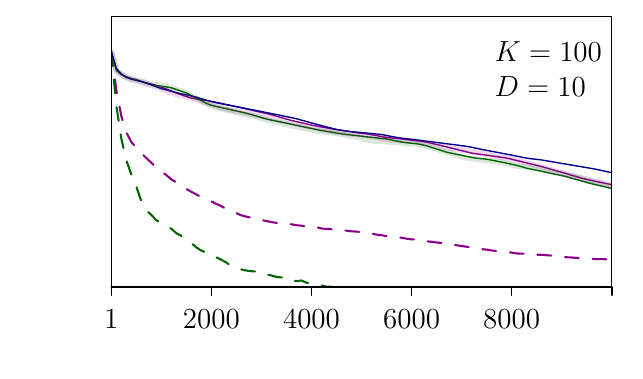}

\caption{\gls{ADAM} with regularised weights.}
\label{fig:nondiagonal:adam_regularised_weights}
\end{subfigure}

    \caption{The same setting as in Figure~\ref{fig:diagonal} except that here, the covariance matrix \emph{$\smash{\varSigma = (0.95^{\lvert d-e\rvert +1})_{(d,e) \in \{1,\dotsc,D\}^2}}$ is not a diagonal matrix.} Again, note the logarithmic scaling on the second axis.}
  \label{fig:nondiagonal}
\end{figure}
  
\restoregeometry

\paragraph{Summary of results.}
Below, we outline what we believe to be the main takeaways from these simulation results for this particular model. However, further theoretical analysis is required to determine whether these hold in more general scenarios.
\begin{enumerate}

 \item The \gls{KL}-divergence based \gls{AISLE} algorithms typically performed somewhat better than their $\chi^2$-divergence based \gls{AISLE} counterparts, i.e.\ \gls{AISLEKLNOREP} outperformed \gls{AISLECHISQNOREP} while \gls{AISLEKL} outperformed \gls{AISLECHISQ}. We conjecture that this is due to the fact that the $\chi^2$-divergence based variants square the (self-normalised) importance weights which increases the variance of the $\phi$-gradients.

 \item The performance of the $\phi$-gradients \gls{AISLEKLNOREP} and \gls{AISLECHISQNOREP} (which do not use\slash{}need any reparametrisation) typically benefited strongly from moderate (relative to the dimension of the latent variables) increases in the number of particles. When \gls{ADAM} was used (and for larger $K$), these gradients outperformed the `score-function free' $\phi$-gradients \gls{AISLEKL}\slash{}\gls{IWAESTL}, \gls{AISLECHISQ}\slash{}\gls{IWAEDREG} in the scenario shown in Figure~\ref{fig:nondiagonal:adam_standard_weights}. We conjecture that this is due to the fact that the variational family does not include the target distribution in this scenario, i.e.\ $q_\phi \neq \target_\theta$ for any $\phi$, and as a result, the main advantage of the `score-function free' gradients -- i.e.\ the fact that they can potentially achieve zero variance -- cannot be realised. 
 
 \item As expected, the performance of the standard \gls{IWAE} $\phi$-gradient consistently became worse with increasing $K$ (see Figures~\ref{fig:diagonal:sga_standard_weights}, \ref{fig:diagonal:adam_standard_weights}, \ref{fig:nondiagonal:sga_standard_weights} and \ref{fig:nondiagonal:adam_standard_weights}). This can be attributed to the fact that the signal-to-noise ratio of this gradient vanishes as $\bo(K^{-1/2})$ as this gradient constitutes a self-normalised importance-sampling approximation of an integral which is equal to zero (see \citet{rainforth2018tighter} and also Subsection~\ref{subsec:iwae}).
 
 \item More surprisingly, the `score-function free' $\phi$-gradients \gls{AISLEKL}\slash{}\gls{IWAESTL}, \gls{AISLECHISQ}\slash{}\gls{IWAEDREG} (as well as the \gls{AISLECHISQNOREP} gradient in Figure~\ref{fig:diagonal:adam_standard_weights}) did not appear to improve with increasing $K$. Indeed, their performance sometimes became worse. We note that this \emph{cannot} be explained by the signal-to-noise ratio decay (which \citet{rainforth2018tighter} highlighted for the standard \gls{IWAE} $\phi$-gradient) because the `score-function free' $\phi$-gradients do not constitute self-normalised importance-sampling approximations of integrals which are equal to zero. Instead, we conjecture that as discussed in Remark~\ref{rem:minimising_exclusive_divergence_preferable} in this model, the $\smash{\bo(K^{-1})}$ self-normalisation bias of these gradients happens to be beneficial and outweighs the $\smash{\bo(K^{-1/2})}$ standard-deviation decrease obtained from increasing $K$. To counteract this issue, we also regularised the weights in each of these estimators as discussed in Subsection~\ref{subsec:bias-variance_trade-offs}, i.e.\ we replaced $\smash{w_{\psi,x}(z^k) / \sum_{l=1}^K w_{\psi,x}(z^l)}$ by $\smash{w_{\psi,x}^{\alpha^\star}(z^k) / \sum_{l=1}^K w_{\psi,x}^{\alpha^\star}(z^l)}$, where $\smash{\alpha^\star}$ was determined as explained in Subsection~\ref{subsec:bias-variance_trade-offs} with $\eta = 0.8$. Figures~\ref{fig:diagonal:sga_regularised_weights}, \ref{fig:diagonal:adam_regularised_weights}. \ref{fig:nondiagonal:sga_regularised_weights} and \ref{fig:nondiagonal:adam_regularised_weights} show that this regularisation strategy appears to be especially beneficial to the $\chi^2$-divergence based \gls{AISLE} gradients.


 
 \item The `doubly-reparametrised' \gls{RWS}-gradient \gls{RWSDREG} from \citet{tucker2019reparametrised} and given in \eqref{eq:rws:gradient_approximation:variance_reduced:unprincipled} performed well for a moderate to large number of particles $K > 1$ in settings in which the oscillation of the importance-weight function, $w_\psi$, is relatively small (or at least if it becomes small as $\phi$ is optimised).
 However, this requires that $q_{\phi,x}$ can be made very close to $\target_{\theta,x}$ for an appropriate choice of $\phi$ which is typically only possible in low-dimensional settings and if the variational family is sufficiently expressive, i.e.\ in the scenario from Figure~\ref{fig:diagonal}. Otherwise, e.g.\ in dimension $D = 10$ in the scenario from Figures~\ref{fig:nondiagonal}, the performance of \gls{RWSDREG} was worse than that of any \gls{AISLE} variants and also worse than the standard \gls{IWAE} reparametrisation-trick gradient. We conjecture that this is because the variance of the weights is so large that typically one of the self-normalised weights $\smash{w_{\psi,x}(z^k) /  \sum_{l=1}^K w_{\psi,x}(z^l)}$ is numerically equal to $1$ while all the others are numerically equal to $0$. Note that whenever this happens, the \gls{RWSDREG} gradient reduces to a vector of $0$s. Again, Figures~\ref{fig:nondiagonal:sga_regularised_weights} and \ref{fig:nondiagonal:adam_regularised_weights} show that the regularisation strategy from Subsection~\ref{subsec:bias-variance_trade-offs} alleviates this problem (though Figures~\ref{fig:diagonal:adam_standard_weights} and \ref{fig:diagonal:adam_regularised_weights} make it clear that \gls{RWSDREG} does not necessarily benefit from this kind of regularisation under all circumstances).
\end{enumerate}
\end{document}